\documentclass[twoside]{article}

\usepackage[accepted]{aistats2020}
\usepackage{natbib}
\usepackage{amsmath,amssymb,amsfonts,amsthm}
\usepackage{algorithm}
\usepackage{algpseudocode}
\usepackage{dsfont}
\usepackage{color}
\usepackage{graphicx}
\usepackage{subfigure}
\usepackage{enumitem}
\usepackage[font=small]{caption}

\newtheorem{assumption}{Assumption}
\newtheorem{theorem}{Theorem}
\newtheorem{lemma}{Lemma}

\algnewcommand{\LeftComment}[1]{ \(\triangleright\) #1}
\ifodd 0
\newcommand{\cssr}[1]{{\color{red}#1}}
\else
\newcommand{\cssr}[1]{#1}
\fi

\ifodd 0
\newcommand{\cssc}[1]{{\color{red}(Chengshuai: #1)}}
\else
\newcommand{\cssc}[1]{}
\fi

\ifodd 0
\newcommand{\congr}[1]{{\color{blue}#1}}
\else
\newcommand{\congr}[1]{#1}
\fi

\ifodd 0
\newcommand{\congc}[1]{{\color{magenta}(Cong: #1)}}
\else
\newcommand{\congc}[1]{}
\fi

\long\def\appaddress#1{
     \ifnum\statePaper=0
    {
        {\def\and{\unskip\enspace{\rm and}\enspace}%
        \def\And{\end{tabular}\hss \egroup \hskip 1in plus 2fil
                \hbox to 0pt\bgroup\hss \begin{tabular}[t]{c} }%
        \def\AND{\end{tabular}\hss\egroup \hfil\hfil\egroup
                \vskip 0.25in plus 1fil minus 0.125in
                \hbox to \linewidth\bgroup \hfil\hfil
                    \hbox to 0pt  \bgroup \hss \begin{tabular}[t]{c}}
        \def\ANDD{\end{tabular}\hss\egroup \hfil\hfil\egroup
                \vskip 0.25in plus 1fil minus 0.125in
                \hbox to \linewidth \bgroup \hfil\hfil
                    \hbox to 0pt  \bgroup \hss\begin{tabular}[t]{c}\bfseries}
            \hbox to \linewidth\bgroup \hfil\hfil
            \hbox to 0pt\bgroup\hss \begin{tabular}[t]{c} 
            Anonymous Institution
            \end{tabular}
            \hss\egroup
            \hfil\hfil\egroup}
    }
    \else
    {
        {\def\and{\unskip\enspace{\rm and}\enspace}%
        \def\And{\end{tabular}\hss \egroup \hskip 1in plus 2fil
                \hbox to 0pt\bgroup\hss \begin{tabular}[t]{c} }%
        \def\AND{\end{tabular}\hss\egroup \hfil\hfil\egroup
                \vskip 0.25in plus 1fil minus 0.125in
                \hbox to \linewidth\bgroup \hfil\hfil
                    \hbox to 0pt  \bgroup \hss \begin{tabular}[t]{c}}
        \def\ANDD{\end{tabular}\hss\egroup \hfil\hfil\egroup
                \vskip 0.25in plus 1fil minus 0.125in
                \hbox to \linewidth \bgroup \hfil\hfil
                    \hbox to 0pt  \bgroup \hss\begin{tabular}[t]{c}\bfseries}
            \hbox to \linewidth\bgroup \hfil\hfil
            \hbox to 0pt\bgroup\hss \begin{tabular}[t]{c} #1
                                \end{tabular}
        \hss\egroup
        \hfil\hfil\egroup}
    }
   \fi
}

\begin{document}
	
	%
	\runningtitle{Decentralized Multi-player Multi-armed Bandits with No Collision Information}
	
	%
	
	\twocolumn[
	
	\aistatstitle{Decentralized Multi-player Multi-armed Bandits\\ with No Collision Information}
	
	\aistatsauthor{Chengshuai Shi \And Wei Xiong \And Cong Shen \And Jing Yang}
	
	\aistatsaddress{University of Virginia \And University of Virginia \And University of Virginia  \And  Pennsylvania State University} ]

	\begin{abstract}
	The decentralized stochastic multi-player multi-armed bandit (MP-MAB) problem, where the collision information is not available to the players, is studied in this paper. Building on the seminal work of \cite{Boursier2019}, we propose error correction synchronization involving communication (EC-SIC), whose regret is shown to approach that of the centralized stochastic MP-MAB with collision information. By recognizing that the communication phase without collision information corresponds to the \emph{Z-channel} model in information theory, the proposed EC-SIC algorithm applies optimal error correction coding for \congr{the} communication of reward statistics. A fixed message length, as opposed to the logarithmically growing one in  \cite{Boursier2019}, also plays a crucial role in controlling the communication loss. Experiments with practical Z-channel codes, such as repetition code, flip code and \cssr{modified} Hamming code, \congr{demonstrate the superiority of EC-SIC in both synthetic and real-world datasets}. 
	
	

	\end{abstract}
	
	\section{Introduction}
	\label{sec:intro}
	
	\begin{table*}
		\tiny
		\caption{Regret Upper Bounds of MP-MAB Algorithms}
		\vspace{-0.15in}
		\begin{center}
			\begin{tabular}{|ccc|}
				\hline
				\textbf{Model}  &\textbf{Reference} &\textbf{Asymptotic Upper Bound (up to constant factor)} \\
				\hline 
				Centralized Multiplayer &\cite{komiyama2015optimal} &$\sum_{k>M}\frac{\log(T)}{\mu_{(M)}-\mu_{(k)}}$\\
				Decentralized, Col. Sensing &SIC-MMAB \citep{Boursier2019} &$\sum_{k>M}\frac{\log(T)}{\mu_{(M)}-\mu_{(k)}}+MK\log(T)$ \\
				Decentralized, No Sensing &\cite{lugosi2018multiplayer} &$\frac{MK\log(T)}{\Delta^2}$\\
				Decentralized, No Sensing &\cite{lugosi2018multiplayer} &$\frac{MK^2}{\mu_{(M)}}\log^2(T)+MK\frac{\log(T)}{\Delta'}$\\
				Decentralized, No Sensing &ADAPTED SIC-MMAB \citep{Boursier2019} &$\sum_{k>M}\frac{\log(T)}{\mu_{(M)}-\mu_{(k)}}+\frac{M^3K\log(T)}{\mu_{(K)}}\log^2(\log(T))$ \\
				Decentralized, No Sensing &SIC-MMAB2 \citep{Boursier2019} &$M\sum_{k>M}\frac{\log(T)}{\mu_{(M)}-\mu_{(k)}}+\frac{MK^2}{\mu_{(K)}}\log(T)$\\
				Decentralized, No Sensing &EC-SIC \congr{(this paper)}  &$\sum_{k>M}\frac{\log(T)}{\mu_{(M)}-\mu_{(k)}}+( \frac{M^2K}{E\cssr{(\mu_{(K)})}}\log(\frac{1}{\Delta})+\frac{MK}{\mu_{(K)}})\log(T)$ \\
				\hline
			\end{tabular}
		\end{center}
		\begin{center}
		     $K$: number of arms; $M\cssr{\leq}K$: number of players; $\mu_{(j)}$ is $j$-th order statistics of $\mu$; $\Delta := \mu_{(M)}-\mu_{(M+1)}>0$; 
		     $\Delta':=\min\{\mu_{(M)}-\mu_{(i)}|\mu_{(M)}-\mu_{(i)}>0\}$; $E(\mu_{(K)})$ is the random coding error exponent.
		\end{center}
		\vspace{-0.2in}
		\label{regret_table}
	\end{table*}  	
	
	Recent years have \congr{witnessed} an increased interest in the multi-player multi-armed bandits (MP-MAB) problem, in which multiple players simultaneously play the bandit game and interact with each other through arm collisions. In particular, motivated by the practical application of cognitive radio \citep{anandkumar2011distributed}, \emph{decentralized} stochastic MP-MAB problems have been widely studied. See Section~\ref{sec:related} for a review of the related work.
	
	Since the stochastic MAB problem for a single player is well understood, a predominant approach in decentralized MP-MAB is to let each player play the single-player MAB game while avoiding collisions for as much as possible \citep{liu2010distributed,avner2014concurrent,rosenski2016multi,besson2017multi}. Intuitively, this would allow the algorithm to \congr{behave as the single-player MAB}. Comparing to the centralized MP-MAB, there is a multiplicative factor $M$ (number of players) increase in the regret coefficient \congr{of $\log(T)$}. This has long been considered fundamental due to the lack of communication among players.
	
    Recently, a surprising and inspiring approach, called SIC-MMAB, is proposed in \cite{Boursier2019} for the collision-sensing \congr{MP-MAB} problem. Instead of viewing collision as detrimental, \cite{Boursier2019} purposely instigates collisions as a way to communicate between players. With a careful design, an internal rank can be assigned to each player and arm statistics can be completely shared among players at a communication cost that does not dominate the arm \congr{exploration} regret, which leads to an overall regret approaching the centralized setting. The proposed communication phase transmits the \emph{total reward} by using collision/no-collision to represent bit 1/0. The theoretical analysis shows, for the first time, that the regret of a decentralized MP-MAB algorithm can approach that of the centralized counterpart, which represents a significant progress in \congr{decentralized} MP-MAB.
    
    The no-sensing problem, on the other hand, represents arguably the most difficult setting in MP-MAB and there has been little progress in the literature. \cite{Boursier2019} makes two attempts to generalize the forced collision idea to this setting.  Directly applying SIC-MMAB to this setting leads to a O($\log(T)\log^2(\log(T))$) regret, \congr{which has an additional $\log(T)$ multiplicative coefficient due to no sensing}. In other words, a straightforward application of SIC-MMAB results in the communication loss dominating the total regret. Then, the authors propose a different approach: use communication only to exchange the accepted and rejected arms, thus reducing the regret caused by communication. However, this approach, philosophically speaking, deviates from the core idea of SIC-MMAB and does not fully utilize the communication benefit of collision (arm statistics are not shared among players). This is also the reason why the multiplicative factor $M$ reappears in the regret formula, which was eliminated in the collision-sensing case by SIC-MMAB. It remains an open problem \emph{whether a decentralized MP-MAB algorithm without collision information can approach the performance of its centralized counterpart}.

	In this work, we return to the original idea of utilizing collisions to communicate sampled arm rewards. By modelling no collision information as the \emph{Z-channel} communication problem in information theory, we propose to incorporate optimal error correction coding in the communication phase to control \congr{the error rate of decoding the message}. With this approach, we are able to transmit a \emph{quantized} sample reward mean with a \emph{fixed} length for each player without having the communication loss dominate the total regret. 
	The resulting asymptotic regret improves the $\log(T)$ coefficients over SIC-MMAB2, and represents the best known regret in the no-sensing MP-MAB model to the best of the authors' knowledge.
	Table~\ref{regret_table} compares the asymptotic regret upper bounds for different algorithms. We also propose two practical enhancements that significantly improve the algorithm's empirical performance. \congr{Numerical experiments on both synthetic and real-world datasets corroborate the analysis and offer interesting insights into EC-SIC.}

	\vspace{-0.1in}
	\section{The No-Sensing MP-MAB Problem}
	\label{sec:model}
	In the standard (single player) stochastic MAB setting, there are $K$ arms, with rewards $X_i$ of arm $i\in [K]$ sampled independently from \congr{a} distribution $\nu_i$ on $[0,1]$, \congr{where} $\mathbb{E} \left[ X_i \right] = \mu[i]$. At time $t$, a player chooses an arm $\pi(t)$ and the goal is to receive the highest \congr{mean} cumulative reward in $T$ rounds.
	
	In this section, we introduce the no-sensing multiplayer MAB model with a known number of arms $K$ but an unknown number of players $M\leq K$. The horizon $T$ is known to the players. At each time step $t\in [T]$, all the players $j\in [M]$ simultaneously pull the arms $\pi^j(t)$ and receive the reward $r^j(t)$ such that 
    \begin{equation*}
    r^j(t):=X_{\pi^j(t)}(t)(1-\eta_{\pi^j(t)}(t))
    \end{equation*}
	where $\eta_{\pi^j (t)}(t)$ is the collision indicator defined by
	$\eta_{k}(t):=\mathds{1}_{ | C_k(t) |>1}$
	with $C_k(t):=\{j\in[M]|\pi^j(t)=k\}$. 
	
	If players can observe both $r^j(t)$ and $\eta_{\pi^j(t)}\cssr{(t)}$, it is the collision-sensing problem. On the other hand, in the no-sensing case as in our paper, players can only access $r^j(t)$, i.e., a reward of $0$ can indistinguishably come from a collision with another player or $X_{\pi^j(t)}(t)=0$.
	Note that if $\mathbb{P}(\forall k\in[K], X_k=0) =0$, the no-sensing and collision-sensing models are equivalent\footnote{We further note that the no-sensing model can be generalized to an arbitrary but bounded reward support where collision results in the lowest value in the support.}.
	
	The performance in the standard single-player MAB setting is usually measured by the regret:
	\vspace{-0.1in}
    \begin{equation*}
    R(T):=T \mu_*-\sum_{t=1}^{T}\mu[\pi(t)],
    \end{equation*}
	where $\mu_*=\max_i\mu[i]$ is the expected reward of the arm with the highest expected reward. As shown in the lower bound by \cite{lai1985asymptotically}, the optimal order of the regret cannot be better than $O(\log(T))$.
	
	In the multiplayer setting, the notion of regret can be generalized and defined with respect to the best allocation of players to arms, as follows:
	\vspace{-0.1in}
	$$R(T):=T\sum_{j\in [M]}\mu_{(j)}-\mathbb{E}\left[\sum_{t=1}^{T}\sum_{j\in [M]}r^j(t) \right],$$
	where $\mu_{(j)}$ is $j$-th order statistics of $\mu$, i.e. $\mu_{(1)}\geq\mu_{(2)}\geq...\geq\mu_{(K-1)}\geq\mu_{(K)}$.
	
	Two technical assumptions are made in this paper, which are also widely used in the literature. The first is a strictly positive lower bound of $\mu_{(K)}$, which has been used by \cite{lugosi2018multiplayer} and \cite{Boursier2019} for the no-sensing model. The second assumption is a finite gap between the optimal and suboptimal (group of) arms; see \cite{avner2014concurrent,kalathil2014decentralized,rosenski2016multi,nayyar2016regret}.
	\begin{assumption}
	\begin{enumerate}[leftmargin=12pt,topsep=0pt, itemsep=0pt,parsep=0pt]
	\item[1)] 
	A positive lower bound $\mu_{\min}$ of $\mu_{(K)}$ is known to all players: $0<\mu_{\min} \leq \min_{i \in [K]} \mu[i]$.
	\item[2)] 
	There exists a positive gap $\Delta \doteq\mu_{(M)}-\mu_{(M+1)}> 0$, and it is known to all players.
	\end{enumerate}
	\label{asp}
	\end{assumption}
	Assumption~1.1 is equivalent to $\forall k\in[K]$, $\mathbb{P}[X_k>0]\geq \mu_{\min}$. This also bounds $\mathbb{P}[X_k=0]$. \cssr{Note that although $\mu_{\min}$ provides a lower bound for $\mu_{(K)}$, Assumption 1.1 does not require the exact value of $\mu_{(K)}$.} The gap in Assumption~1.2 measures the difficulty of the bandit game and ensures the existence of only one optimal choice.
	
	\vspace{-0.1in}
	\section{The EC-SIC Algorithm}
	\label{sec:alg}
	The proposed error correction synchronization involving communication (EC-SIC) is compactly described in Algorithm~\ref{alg_overall}. Similar to SIC-MMAB, the overall algorithm can be \congr{structurally} divided into four phases: initialization phase, exploration phase, communication phase, and exploitation phase. It is important to note that all players are synchronized in running EC-SIC, i.e., they enter each phase at the same time (or at least with high probability in some cases) except the exploitation phase. Until a player fixates on a specific arm and enters the exploitation phase, the algorithm keeps iterating between the exploration and communication phases. Players that have (not) entered the exploitation phase are called inactive (active). We denote the set of active players during the $p$-th phase by $[M_p]$ and its cardinality by $M_p$. Similarly, arms that have not been decided to be optimal or sub-optimal are called active. The set of active arms during the $p$-th phase is denoted by $[K_p]$ with cardinality $K_p$. 
	
	\begin{algorithm}[thb]
		\small
		\caption{The EC-SIC Algorithm}
		\label{alg_overall}
		\begin{algorithmic}[1]
			\Require $T,\ K,\ \Delta,\ \epsilon,\ \mu_{\min}$;
			\State Initialize \cssr{$p \gets 1$}; $F \gets -1$; $T_0, T_0^j\gets 0$; $[K_p] \gets [K]$; \textcolor{black}{$Q\gets \max\{\lceil\log_2{\frac{1}{\frac{\Delta}{4}-\epsilon}} \rceil, \lceil\log_2(K+1)\rceil\}$}; $T_c\gets \lceil \frac{\log(T)}{\mu_{\min}}\rceil$
			\State Select an error-correction code $\left(\cssr{N', Q}\right)$ with code length $N'$ defined in Theorem~\ref{regret_overall_theorem}
			\Statex\LeftComment{\texttt{Initialization Phase:}}
			\State $k \gets$ Musical\_Chair($[K]$, $KT_c$)
			\State $(M,\ j) \gets$ Estimate\_M\_NoSensing($\cssr{k}$, $T_c$)
			\While{$F=-1$}
			\Statex \LeftComment{\texttt{Exploration Phase:}}
			\State $\pi \gets j$-th active arm
			\For {\textcolor{black}{$K_p 2^p \lceil \log(T) \rceil $ time steps}}
			\State$\pi \gets \pi+1\ (\text{mod }K_p)$ and play arm $\pi$
			\State $s[\pi]\gets s[\pi]+\cssr{r^j(t)}$
			\EndFor
			\State $T_p^j=T_{p-1}^j+2^p\lceil\log(T)\rceil$
			\State \textcolor{black}{$\hat{\mu}_j=s/T_p^j$}
			\Statex \LeftComment{\texttt{Communication Phase:}}
			\If {$j=1$} ($F$, $M_{p+1}$, [$K_{p+1}$]) $\gets$ \textcolor{black}{Communication Leader}($\textcolor{black}{\cssr{\hat{\mu}_1}},\ p,\ [K_p],\ M_p,\ Q,\ N'$)
			\Else \ ($F$, $M_{p+1}$, [$K_{p+1}$]) $\gets$ \textcolor{black}{Communication Follower}($\textcolor{black}{\hat{\mu}_j},\ j,\ p,\ [K_p],\ M_p,\ Q,\ N'$)
			\EndIf
			\State $p\gets p+1$
			\EndWhile
			\State \LeftComment{\texttt{Exploitation phase: }}Pull $F$ until $T$
		\end{algorithmic}
	\end{algorithm}

	\subsection{Initialization phase}
	The same structure of \congr{the} initialization phase as \cite{Boursier2019} is used in EC-SIC, which outputs an internal rank $j\in\{1,......,M\}$ for each player as well as the \congr{estimated} value of $M$. It starts with a ``\congr{Musical} Chair'' phase and is followed by a so-called Sequential Hopping protocol. The full procedure is described in Appendix~\ref{appendix_proof_init} for completeness. 
	
	\subsection{Exploration phase}
	During the $p$-th exploration phase, active players sequentially hop among the active arms for $K_p2^p\lceil\log(T)\rceil$ steps, and any active arm is pulled $2^p\lceil\log(T)\rceil$ times by each active player. Since the hopping is based on each player's internal rank, the exploration phase is collision-free. 
	
	We note that the length of an exploration phase is different from \cite{Boursier2019}, which is a key component of the performance improvement. The difference of a $\lceil\log(T)\rceil$ factor, in fact, results in an overall $\cssr{O(}\log(\log(T))\cssr{)}$ rounds of exploration and communication phases in the ADAPTED SIC-MMAB algorithm of \cite{Boursier2019}. This directly leads to a dominating communication loss that breaks the order-optimality. With an expansion of length by $\lceil\log(T)\rceil$ in EC-SIC, the overall rounds become a constant, \congr{and the communication regret can be better controlled as shown in Section~\ref{sec:analysis}.}
	

	\begin{algorithm}[htb]
		\small
		\caption{Communication Leader}
		\label{alg_comm_leader}
		\begin{algorithmic}[1]
			\Require $\hat{\mu}_1,\ p,\ [K_p],\ M_p,\ Q,\ N'$
			\Ensure $F$, $M_{p+1}$, [$K_{p+1}$]
			\State Initialize $T_p=T_{p-1}+M_p2^p\lceil\log(T)\rceil$; \cssr{$T_p^i= T_p^1$, $i=1,...,M_p$; $\bar{\mu}_i^p=\bar{\mu}^{p-1}_i$, $T_p^i=T_{p-1}^i$, $i=M_p+1,...,M$}
            \Statex \LeftComment{\texttt{Gather information from followers:}}
			\For {$i=2,...,M_p$}\Comment{\cssr{\texttt{Receive arm statistics}}}
			\For {$k\in[K_p]$}
			\State $\textcolor{black}{\bar{\mu}^p_{i}[k]}\gets$ \textcolor{black}{Decoder}(Receive(\cssr{$1$}, $i$, $N'$)) 
			\EndFor 
			\EndFor
			\State $\bar{\mu}^{p}=\sum_{i=1}^{M}\bar{\mu}^p_{i}\cssr{\cdot T_p^i/T_p}$; \textcolor{black}{$B_{T_p}=\sqrt{\frac{2\log(T)}{T_p}}+(\frac{\Delta}{4}-\epsilon)$}
			\Statex \LeftComment{\texttt{Update statistics:}}
			\State Rej $\gets$ set of active arms $k$ \congr{satisfying} $|\{i\in[K_p]|\bar{\mu}^p[i]-B_{T_p}\geq \bar{\mu}^p[k]+B_{T_p}\} |\geq M_p$ 
			\State Acc $\gets$ set of active arms $k$ \congr{satisfying} $| \{i\in[K_p]|\bar{\mu}^p[k]-B_{T_p}\geq \bar{\mu}^p[i]+B_{T_p}\} |\geq K_p-M_p$, ordered according to their indices
			\Statex \LeftComment{\texttt{Transmit acc$\backslash$rej arms to followers:}}
			\For {$i=2,..., M_p$}\Comment{\cssr{\texttt{Send acc$\backslash$rej set size}}}
			\State Send(\cssr{$1$}, $i$, $N'$, \textcolor{black}{Encoder}($|\text{Rej}|$, $Q$))
			\State Send(\cssr{$1$}, $i$, $N'$, \textcolor{black}{Encoder}($|\text{Acc}|$, $Q$))
			\EndFor
			\For {$i=2,...,M_p$}\Comment{\cssr{\texttt{Send acc$\backslash$rej set content}}}
			\State Send(\cssr{$1$}, $i$, $N'$, Encoder($k$, $Q$)) for $k\in \text{Rej}$
			\State Send(\cssr{$1$}, $i$, $N'$, Encoder($k$, $Q$)) for $k\in \text{Acc}$
			\EndFor
			\If {$\cssr{M_p} \leq |\text{Acc}|$} 
			\State $F \gets \text{Acc}[\cssr{M_p}]$
			\Else \ $M_p\gets M_p-|\text{Acc}|$
			\State $[K_{p+1}]\gets[K_p]\backslash(\text{Acc}\cup\text{Rej})$
			\EndIf
		\end{algorithmic}
	\end{algorithm}
	
	\begin{algorithm}[htb]
		\small
		\caption{Communication Follower}
		\label{alg_comm_follower}
		\begin{algorithmic}[1]
			\Require $\hat{\mu}_j,\ j,\ p,\ [K_p],\ M_p,\ Q,\ N'$
			\Ensure $F$, $M_{p+1}$, [$K_{p+1}$]
			\Statex \LeftComment{\texttt{Transmit information to the leader:}}
			\For {$i=2,...,M_p$}\Comment{\cssr{\texttt{Send arm statisitcs}}}
			\If {$j=i$}
			\State Send($j$, $1$, $N'$, \textcolor{black}{ Encoder}(\textcolor{black}{$\hat{\mu}_j[k]$}, \cssr{$Q$})) for $k\in[K_p]$
			\Else  \ pull the $j$-th active arm for $K_pN'$ steps
			\EndIf
			\EndFor
			\Statex\LeftComment{\texttt{Receive acc$\backslash$rej arms from the leader:}}
			\For {$i=2,...,M_p$}\Comment\cssr{{\texttt{Receive acc$\backslash$rej set size}}}
			\If {$j=i$} 
			\State $N_{\text{rej}}\gets$ \textcolor{black}{Decoder}(Receive($j$, $1$, $N'$))
			\State $N_{\text{acc}}\gets$ \textcolor{black}{Decoder}(Receive($j$, $1$, $N'$))
			\Else \ pull $j$-th active arm for $2N'$ steps
			\EndIf
			\EndFor
			\For {\cssr{$i=2,...,M_p$}}\Comment{\cssr{\texttt{Receive acc$\backslash$rej set content}}}
			\If {$j=i$}
			\State $\boldsymbol{w}[k]\gets$Receive($j,1,N'$)) and 
			\State Rej[$k$] $\gets$ \textcolor{black}{Decoder}($\boldsymbol{w}[k]$) for $k=1,...,N_{\text{rej}}$
			\State $\boldsymbol{w}[k]\gets$Receive($j,1,N'$)) and 
			\State Acc[$k$] $\gets$ \textcolor{black}{Decoder}($\boldsymbol{w}[k]$) for $k=1,...,N_{\text{acc}}$
			\Else \ pull $j$-th active arm for $(N_{\text{rej}}+N_{\text{acc}})N'$ steps
			\EndIf
			\EndFor
			\If {$M_p-j+1 \leq |\text{Acc}|$} 
			\State $F \gets \text{Acc}[M_p-j+1]$
			\Else \ $M_p\gets M_p-|\text{Acc}|$
			\State $[K_{p+1}]\gets[K_p]\backslash(\text{Acc}\cup\text{Rej})$
			\EndIf
		\end{algorithmic}
	\end{algorithm}

	\subsection{Communication phase}
	In the communication phase, all players attempt to exchange their sampled reward information in a synchronized and distributed manner. The communication \congr{takes place} via a careful collision design. All players enter this phase synchronously and, by default, keep pulling different arms based on their internal ranks. Then, when it is player $i$'s turn to communicate with player $j$, she would purposely pull (not pull) player $j$'s arm as a way to communicate bit 1 (0). If player $j$ can fully access the collision information, i.e., knowing whether collision happens or not at each time step, she will be able to receive the bit sequence successfully, which conveys player $i$'s sample reward statistics. However, for the no-sensing model, such error-free communication becomes impossible.
	
	Three new ideas are used in the communication phase of EC-SIC. The first is the introduction of \emph{Z-channel coding}. 	
	In the no-sensing scenario, players cannot directly identify collision. If the same communication protocol in \cite{Boursier2019} (representing $1$ or $0$ by collision or no collision) is used, the confusion may mislead the player to believe that collision has occurred (bit $1$) while it is actually a null statistic of reward sampling (bit $0$). This error has a catastrophic consequence in that it breaks the essential synchronization between players. We are thus facing the challenge of communicating the reward statistics to other users while controlling the error rate for the overall communication loss to \emph{not} dominate the regret.
	
	\begin{figure}[thb]
		\vspace{-0.1in}
		\centering
		\includegraphics[width=0.85 \linewidth]{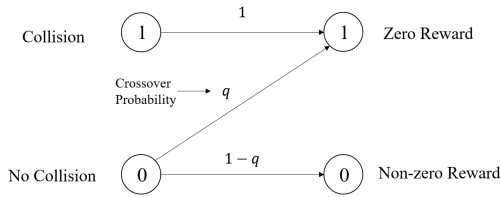}
		\caption{The Z-channel model}
		\label{zchannel}
		\vspace{-0.15in}
	\end{figure}
	
	Luckily, this is the well-known \emph{reliable communication over a noisy channel} problem, one of the foundations in information theory. In particular, our communication channel is {asymmetric}: $1$ (collision) is always received correctly and $0$ (no collision) may be received incorrectly with a certain probability $P(X_k=0)$. This corresponds to the \textbf{Z-channel} model (see Fig.\ref{zchannel}) in information theory \citep{Tallini2002}, which represents a broad class of asymmetric channels. The Z-channel has a crossover probability \cssr{$q$} of $0 \rightarrow 1$ that corresponds to $P(X_k=0)$\footnote{\cssr{Since the crossover probability $P(X_k=0)$ is unknown and varies for different arm $k$, $1-\mu_{\min}$ is used to capture the worst case.}}. 
	
	The Z-channel capacity is derived in \cite{Tallini2002} as follows.
	\begin{theorem}
		The capacity $C_Z(\cssr{q})$ of a Z-channel with crossover $0\to 1$ probability $\cssr{q}$ is:
		\begin{equation}
		\cssr{C_Z(q)=\log_2(1+(1-q)q^{q/(1-q)})}.
		\label{eqn:zchcap}
		\end{equation}
	\end{theorem}
	Shannon theory guarantees that as long as the coding rate $R$ is below $C_Z(\cssr{q})$ in Eqn.~\eqref{eqn:zchcap}, there exists at least one code that allows for an arbitrarily low error rate asymptotically. This means that theoretically, for this Z-channel, it is possible to transmit information nearly \congr{error-free} when the rate is close to $C_Z(\cssr{q})$ bits per channel use.  In reality, however, different finite block-length channel codes may have different performances; we thus evaluate several practical codes both theoretically (in Section \ref{sec:analysis}) and experimentally (in Section \ref{sec:sim}).  For simplicity, Functions $\texttt{Send()}$, $\texttt{Receive()}$, $\texttt{Encoder()}$ and $\texttt{Decoder()}$ are used in the algorithm as the sending and receiving protocol and the encoder and corresponding decoder, respectively.

	The second enhancement is to \emph{transmit each arm's quantized sample reward mean with a fixed length}. The reason not to use the \emph{total} reward as \cite{Boursier2019} is that the gradually increasing total reward leads to a message length $O(\log(\log(T)))$, which cannot be transmitted efficiently in the no-sensing case with $O(\log(T))$ bits. However, with a finite gap, a less precise statistics sharing is tolerable as long as it does not affect the choice of optimal arms. For a quantized sample mean of length $Q$, the error is at most $2^{-Q}$. We thus control the length $Q$ such that $2^{-Q}<\frac{\Delta}{4}-\epsilon$, where $\epsilon\in(0,\frac{\Delta}{4})$ is a pre-defined constant. By constructing $B_s=\sqrt{\frac{2\log(T)}{s}}+(\frac{\Delta}{4}-\epsilon)$ as the confidence bound, analysis in Section~\ref{sec:analysis} shows that acceptation and rejection maintain a high probability of success.
	
	Lastly, compared to the mesh-structured communication in \cite{Boursier2019}, it is more efficient to form a tree structure that one player (``leader'') gathers all statistics and makes decisions for others (``followers''). The player with internal rank $1$ becomes the leader and the rest become the followers \citep{kaufmann2019new}. Statistics of arms are transmitted from followers to the leader. The leader decides the set of arms to be accepted or rejected \cssr{by comparing their upper confidence bounds and lower confidence bounds with each other}, and sends back to the followers. Upon reception, active players either enter another iteration of exploration and communication, or begin exploitation. This process utilizes reward statistics from all players and has better communication efficiency. Procedures for the leader and followers are given in Algorithm \ref{alg_comm_leader} and \ref{alg_comm_follower}, respectively.
		
	\vspace{-0.1in}
	\section{Theoretical Analysis}
	\label{sec:analysis}
	The overall regret of EC-SIC can be decomposed as $R(T)=R^{init}+R^{expl}+R^{comm}$. The first, second and third term refers to the regret caused by the initialization, exploration, and communication phase, respectively. The main result is presented in Theorem \ref{regret_overall_theorem}, and each component regret is subsequently analyzed. Detailed proofs can be found in Appendix~\ref{appendix_proof}.
	
	\begin{theorem}
		\label{regret_overall_theorem}
		With an optimal coding technique that achieves Gallager's error exponent $E(\mu_{\text{min}})$ for the corresponding Z-channel with crossover probability $1-\mu_{\text{min}}$, for any $\epsilon\in (0,\frac{\Delta}{4})$, we have
		\vspace{-0.07in}
		\begin{equation}
		\small
		\begin{aligned}
		&R(T) \leq c_1 MK\frac{\log(T)}{\mu_{\min}}\\
		&+ c_2 \frac{\Delta}{4\epsilon} \left (\sum_{k>M} {\min} \left \{\frac{\log(T)}{\mu_{(M+1)}-\mu_{(k)}+4\epsilon},\sqrt{T\log(T)} \right \} \right ) \\
		&+ c_3 N'\left  (M^2(K+2)\log \left ( \min \left \{\frac{1}{4\epsilon},\cssr{T} \right \} \right) +M^2K \right)
		\end{aligned}
		\label{eqn:regretoverall}
		\end{equation}
		where $c_1$, $c_2$ and $c_3$ are constants and $N'=\max\{\frac{Q}{\cssr{C_Z(1-\mu_{\min})}}, \frac{1}{E(\mu_{\text{min}})}\log(T)\}$. 
	\end{theorem}
	Theorem~\ref{regret_overall_theorem} involves an information-theoretic concept called \emph{error exponent}, which is explained in Theorem~\ref{coding_theorem} in Section~\ref{regret_explo} but more details can be found in \citep{gallager1968information}.
	
	An asymptotic upper bound can be obtained from \eqref{eqn:regretoverall} with $\epsilon=\frac{\Delta}{8}$:
	\vspace{-0.07in}
	\begin{equation}
	\small
	\label{eqn:regretasy}
	\resizebox{.49 \textwidth}{!}{
	$
	R(T) =O \left (\sum_{k>M}\frac{\log(T)}{\mu_{(M)}-\mu_{(k)}}+(\frac{M^2K\log(\frac{1}{\Delta})}{E(\mu_{\min})}+\frac{MK}{\mu_{\min}})\log(T) \right ).
	$
	}
	\end{equation}
	Compared to SIC-MMAB2, we have successfully removed the multiplicative factor of $M$ in the first $\log(T)$ term. This is due to the efficient communication phase that transmits the reward statistics. In addition, we have a $M^2K$ factor in the second $\log(T)$ term, as opposed to $MK^2$ in SIC-MMAB2. This is also an improvement since $M<K$. \cssr{We also note that Eqn. \eqref{eqn:regretoverall} and Eqn. \eqref{eqn:regretasy} hold when $\mu_{(\min)}$ is replaced by $\mu_{(K)}$.} 
	
	\congr{To prove Theorem~\ref{regret_overall_theorem}, we first define the ``typical event'' as the success of initialization, communication and exploration throughout the entire horizon $T$. More specifically, we define three events: $A_1=\{$each player has a correct estimation of $M$ and an orthogonal internal rank after initialization$\}$; $A_2=\{$messages are decoded correctly in all communication phases$\}$; $A_3=\{|\bar\mu^p[k]-\mu[k]|\leq B_{T_p} \text{ holds for phase $p$}, \forall k\in[\cssr{K_p}], \forall p \}$. We use $P_s$ to denote the probability that the typical event happens, which is $A_1 \cap A_2 \cap A_3$. The regret caused by the ``atypical event'' can be simply bounded by a linear regret $O(MT)$. Then the result of \eqref{eqn:regretoverall} can be proved by controlling $P_s$ to balance both events.  }
	


	\subsection{Initialization phase}
	Similar to Lemma 11 in  \cite{Boursier2019}, we can bound the regret of initialization as follows.
	\begin{lemma}\label{regret_init_lemma}
		With probability $P_i=1-O(\frac{MK}{T})$, event $A_1$ happens.
		Furthermore, the regret of the initialization phase satisfies: 
		\begin{equation*}
		\small
		R^{init}<3MK\left \lceil\frac{\log(T)}{\mu_{\min}} \right \rceil.
		\end{equation*}
	\end{lemma}
	
	\subsection{Exploration phase}\label{regret_explo}
	The regret due to exploration is bounded in Lemma~\ref{regret_explo_lemma}.
	\begin{lemma}
		\label{regret_explo_lemma}
		With probability \cssr{$P_s=1-O(\frac{MK\log(T)}{T})$}, the typical event happens and the exploration regret \congr{conditioned on the typical event} satisfies:
		\begin{equation*}
		\small
		\resizebox{.5 \textwidth}{!}{
		$
		R^{expl} = O \left (\frac{\Delta}{4\epsilon}\sum_{k>M}\min \left \{\frac{\log(T)}{\mu_{(M+1)}-\mu_{(k)}+4\epsilon},\sqrt{T\log(T)} \right \} \right).
		$}
		\end{equation*}
	\end{lemma}
	
	We first present a fundamental result of channel coding for communication in a noisy channel, known as the \emph{error exponent} \citep{gallager1968information}.
	\begin{theorem}
		\label{coding_theorem}
		For a discrete memoryless channel, if $R<C$, there exists a code of block length $N$ without feedback such that the error probability is bounded by
		\begin{equation*}
		\vspace{-0.05in}
		P_e\leq \exp[-N E_{r}(R)],
		\end{equation*}
		where $E_{r}(R)$ is the random coding error exponent with rate $R$.
	\end{theorem}
	
	We note that the error exponent used in Theorem~\ref{regret_overall_theorem} corresponds to $E(\mu_{\text{min}}) = E_{r}(\cssr{C_Z(1-\mu_{\min})})$.
	
	Theorem~\ref{coding_theorem} suggests that, to transmit a $Q$-bit message over a Z-channel, there exists an optimal coding scheme with length $N'=\max\{\frac{Q}{\cssr{C_Z(1-\mu_{\min})}}, N\}$ to achieve an error rate less than $\frac{1}{T}$, where $N = \frac{1}{E(\mu_{\min})}\log(T)$. Several of the existing coding techniques, although not optimal, can achieve this error rate with $N=\Theta(\log(T))$, which only leads to a multiplicative factor larger than $\frac{1}{E(\mu_{\min})}$ but does not change the regret order. For example, with repetition code, flip code and \cssr{modified} Hamming code, we have  $N_{rep}=Q\lceil\frac{\log(QT)}{\mu_{\min}}\rceil$, $N_{flip}=Q \lceil \frac{\log(QT/2)}{\mu_{\min}}\rceil$, $N_{ham} = \frac{7Q}{8}\lceil \frac{\log(7QT/8)}{\mu_{\min}}\rceil$ respectively (see Appendix~\ref{appendix_coding} for detailed analysis of practical codes). The remaining analysis will be based on the optimal channel coding   with the caveat that a ``good'' Z-channel code should be applied in practice. 
	
	With at most $\log(T)$ exploration and communication phases and $K$ arms to be accepted or rejected, there are at most $MK\log(T)$ communication instances on arm statistics, $2M\log(T)$ communication instances on the number of acc/rej arms, and $KM$ communication instances on the index of acc/rej arms. A simple union bound analysis leads to the following result.
	\begin{lemma}
		\label{regret_good_comm}
		Denoting the probability that event \cssr{$A_2$ holds} by $P_r$, with an optimal Z-channel code of $N'=\max\{\frac{Q}{\cssr{C_Z(1-\mu_{\min})}}, \frac{\log(T)}{E(\mu_{\min})}\}$, we have
		\begin{equation*}
		\small
		P_r=\cssr{1-O\left(\frac{MK\log(T)}{T}\right)}.
		\end{equation*}
	\end{lemma}
	
	Lemma~\ref{regret_good_comm} guarantees all communications are correct. To bound the probability that all arms are correctly estimated, we have the following \congr{result}.
	\begin{lemma}
		\label{regret_correct_ar}
		\cssr{In phase $p$, for any active arm $k\in[K_p]$,}
		\begin{equation*}
		\cssr{P\left \{|\bar\mu^p[k]-\mu[k]|\geq B_{T_p}\right \}\leq \frac{2}{T}}.
		\end{equation*}
	\end{lemma}
	
	With at most $\log(T)$ exploration-communication phases, \cssr{event $A_3$ happens }
	with probability:
	\vspace{-0.07in}
	\begin{equation}
	\small
	\label{equation_corret_ar}
	P_c=1-O\left (\frac{K\log(T)}{T} \right ).
	\end{equation} 
	A union bound argument leveraging $P_i$, $P_r$ and $P_c$ leads to probability $P_s$ for the typical event to happen, as defined in Lemma \ref{regret_explo_lemma}. Finally, for the exploration phases, the number of times that an arm is pulled before being accepted or rejected are well controlled.
	\begin{lemma}
		\label{regret_artimes_lemma}
		In the typical event, every optimal arm is accepted after at most $O \left (\frac{\log(T)}{(\mu_{(k)}-\mu_{(M)}+4\epsilon)^2} \right)$ pulls, and every sub-optimal arm is rejected after at most $O \left (\frac{\log(T)}{(\mu_{(M+1)}-\mu_{(k)}+4\epsilon)^2} \right)$ pulls. 
	\end{lemma}
	\cssr{Denote $T^{expl}$ as the overall time of exploration and exploitation phase and $T^{expl}_{(k)}(T)$ as the number of time steps where the $k$-th best arm is pulled during these two phases.}
	With no collision \cssr{in exploration and exploitation}, the exploration regret can be decomposed as \citep{anantharam1987asymptotically}
	\begin{equation}
	\small
	\label{regret_expl_decom}
	\begin{aligned}
	R^{expl}&=\sum_{k>M}(\mu_{(M)}-\mu_{(k)})T^{expl}_{(k)}(T)
	\\&+\sum_{k\leq M}(\mu_{(k)}-\mu_{(M)})(T^{expl}-T^{expl}_{(k)}(T)),
	\end{aligned}
	\end{equation}
     Both components in \eqref{regret_expl_decom} can be upper bounded by Lemma \ref{regret_decom_lemma} in Appendix~\ref{appd:lem6}, which proves Lemma \ref{regret_explo_lemma}.
	
	
    
	\subsection{Communication phase}
	Thanks to the expanded length of each exploration phase and the \congr{fixed}-length quantization of arm statistics, the regret $R^{comm}$ does not dominate the overall regret, as stated in Lemma \ref{regret_comm_lemma}.
	\begin{lemma}
		\small
		\label{regret_comm_lemma}
		In the typical event, 
		\begin{equation*}
		\small
		\resizebox{.5 \textwidth}{!}{
		$
		R^{comm}=O\left(N'\left(M^2\left(K+2\right)\log\left(\min\left\{\frac{1}{4\epsilon},\cssr{T}\right\}\right)+M^2K\right)\right).
		$
		}
		\end{equation*}
	\end{lemma}
	We note that $\log(\min\{\frac{1}{4\epsilon},\cssr{T}\})$ becomes a constant when $T$ is sufficiently large. Noting that $N'=\max\{\frac{Q}{\cssr{C_Z(1-\mu_{\min})}},\frac{\log(T)}{E(\mu_{\min})}\}$, the communication loss has the same order as other phases.
	
	\cssr{\subsection{Overall regret}
	When the typical event happens, the overall regret is bounded by the sum of $R^{init}, R^{comm}$ and $R^{expl}$; otherwise, for the atypical event, the regret can be upper bounded as $MT$. Thus, the overall regret satisfies
	\begin{equation*}
	    R(T)
	    \leq R^{init}+R^{expl}+R^{comm}+O(M^2K\log(T)).
	\end{equation*}
	With Lemmas \ref{regret_init_lemma}, \ref{regret_explo_lemma} and \ref{regret_comm_lemma}, Theorem \ref{regret_overall_theorem} can be proven.
	}
	\begin{figure*}[htb]
	\setlength{\abovecaptionskip}{-2pt} 
	\centering
	\begin{minipage}[t]{0.32\textwidth}
		\centering
		\includegraphics[width=5.5cm,height=4cm]{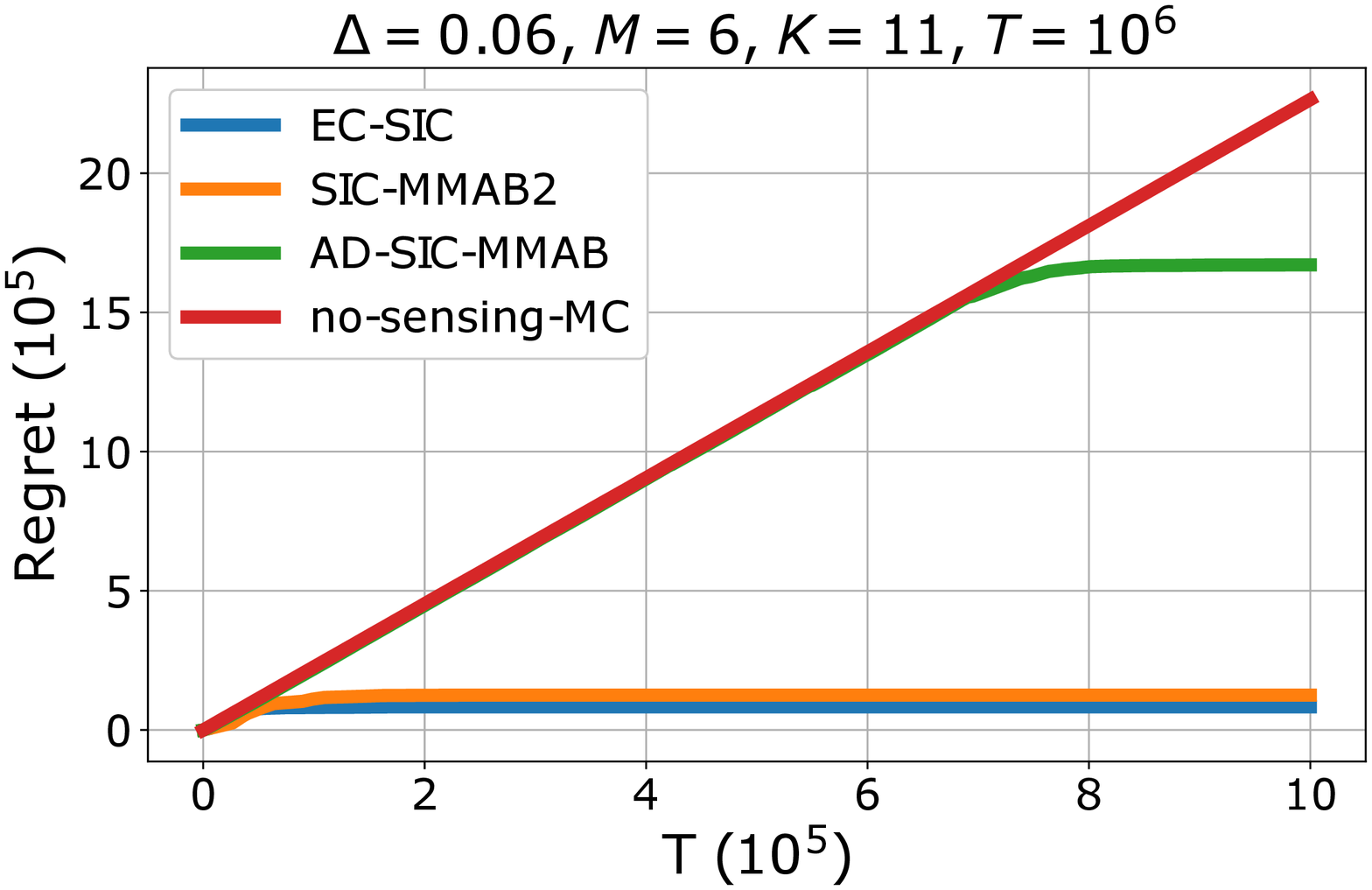}
		\caption{Easy game}
		\label{fig:easy_game}
	\end{minipage}
	\begin{minipage}[t]{0.32\textwidth}
		\centering
		\includegraphics[width=5.5cm,height=4cm]{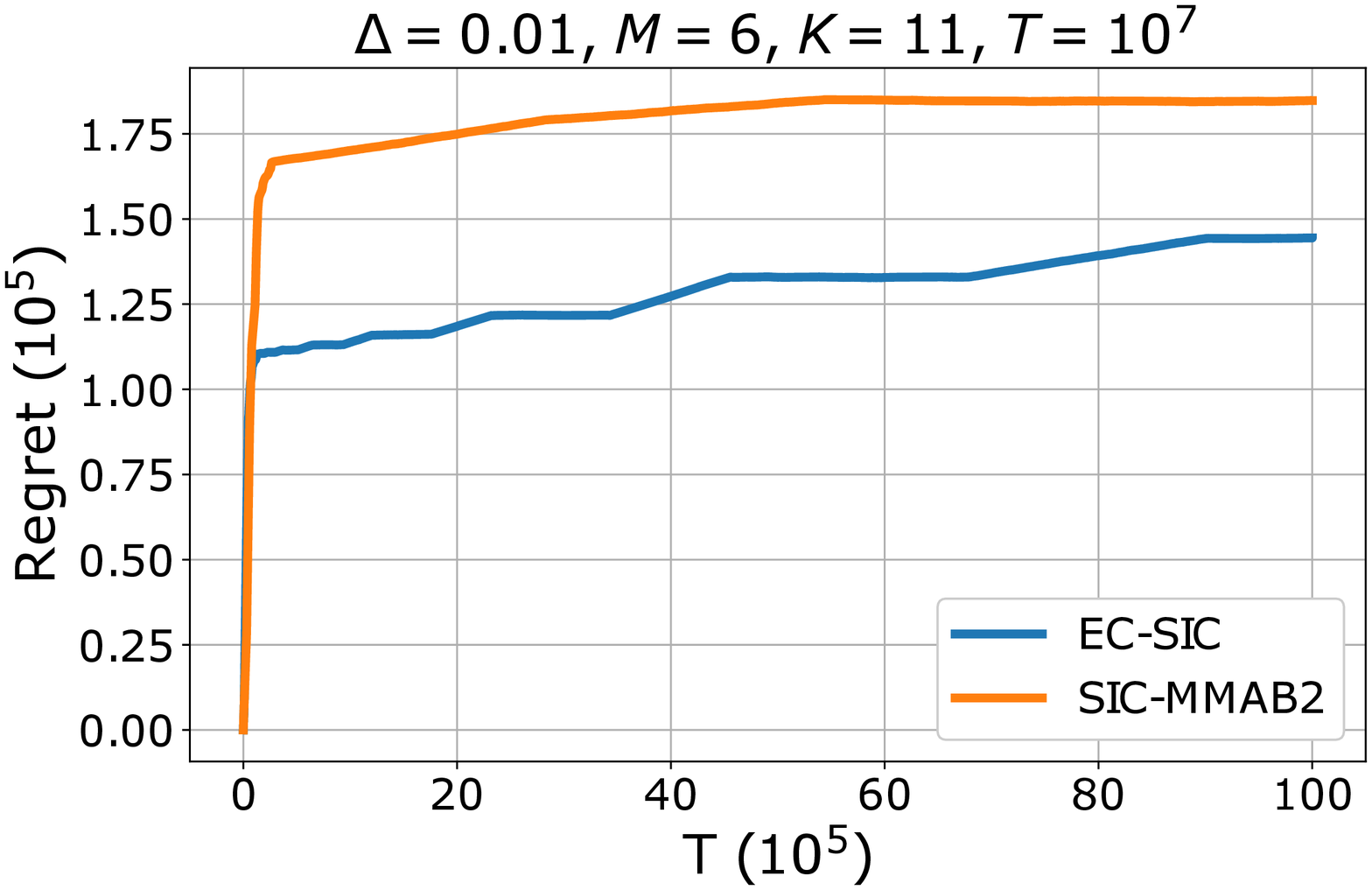}
		\caption{Hard game}
		\label{fig:hard_game}
	\end{minipage}
	\begin{minipage}[t]{0.32\textwidth}
		\centering
		\includegraphics[width=5.5cm,height=4cm]{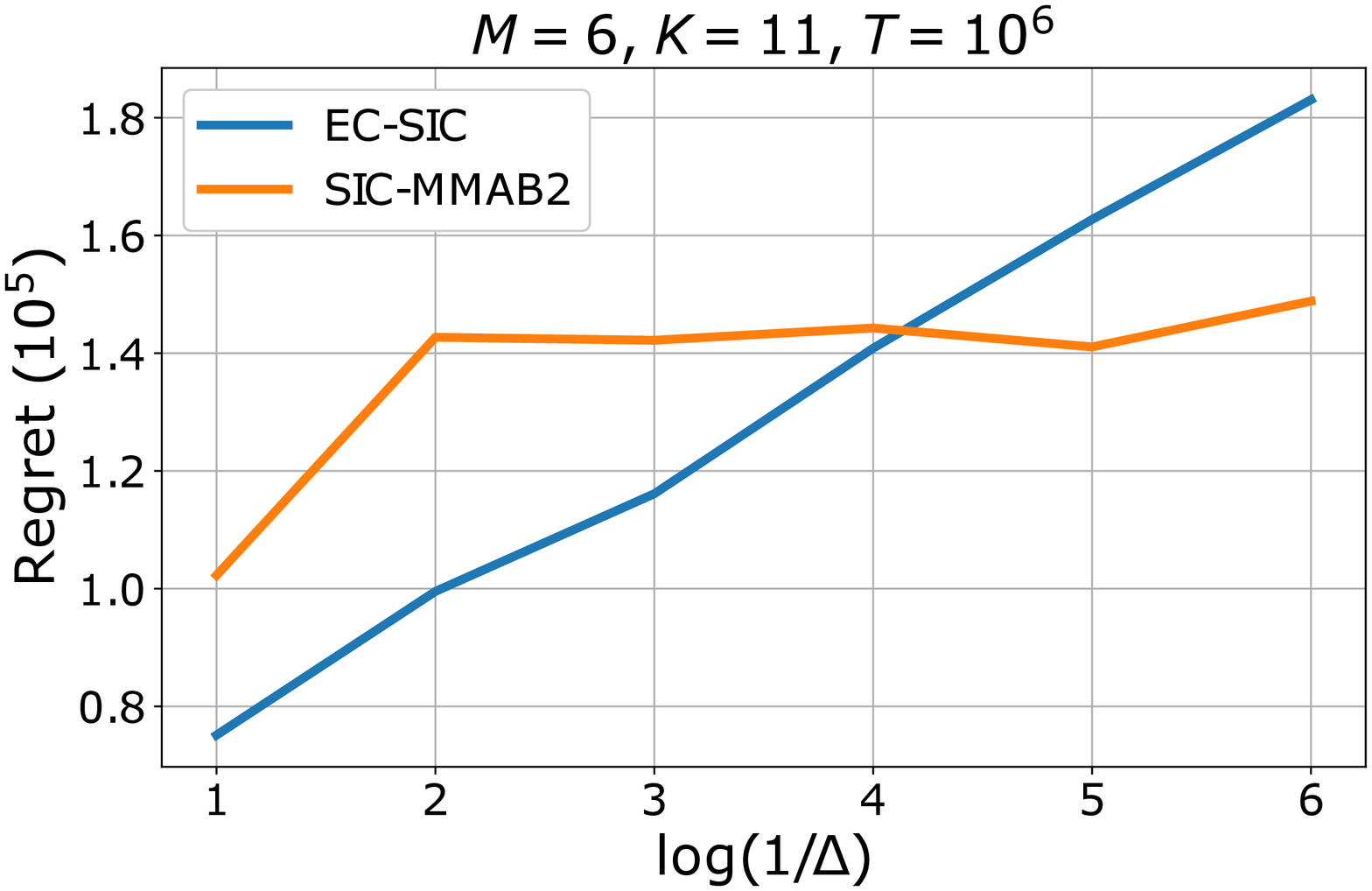}
		\caption{Different game difficulties}
		\label{fig:game_difficulty}
	\end{minipage}
	\vspace{-0.2in}
	\begin{minipage}[t]{0.32\textwidth}
		\centering
		\includegraphics[width=5.5cm,height=4cm]{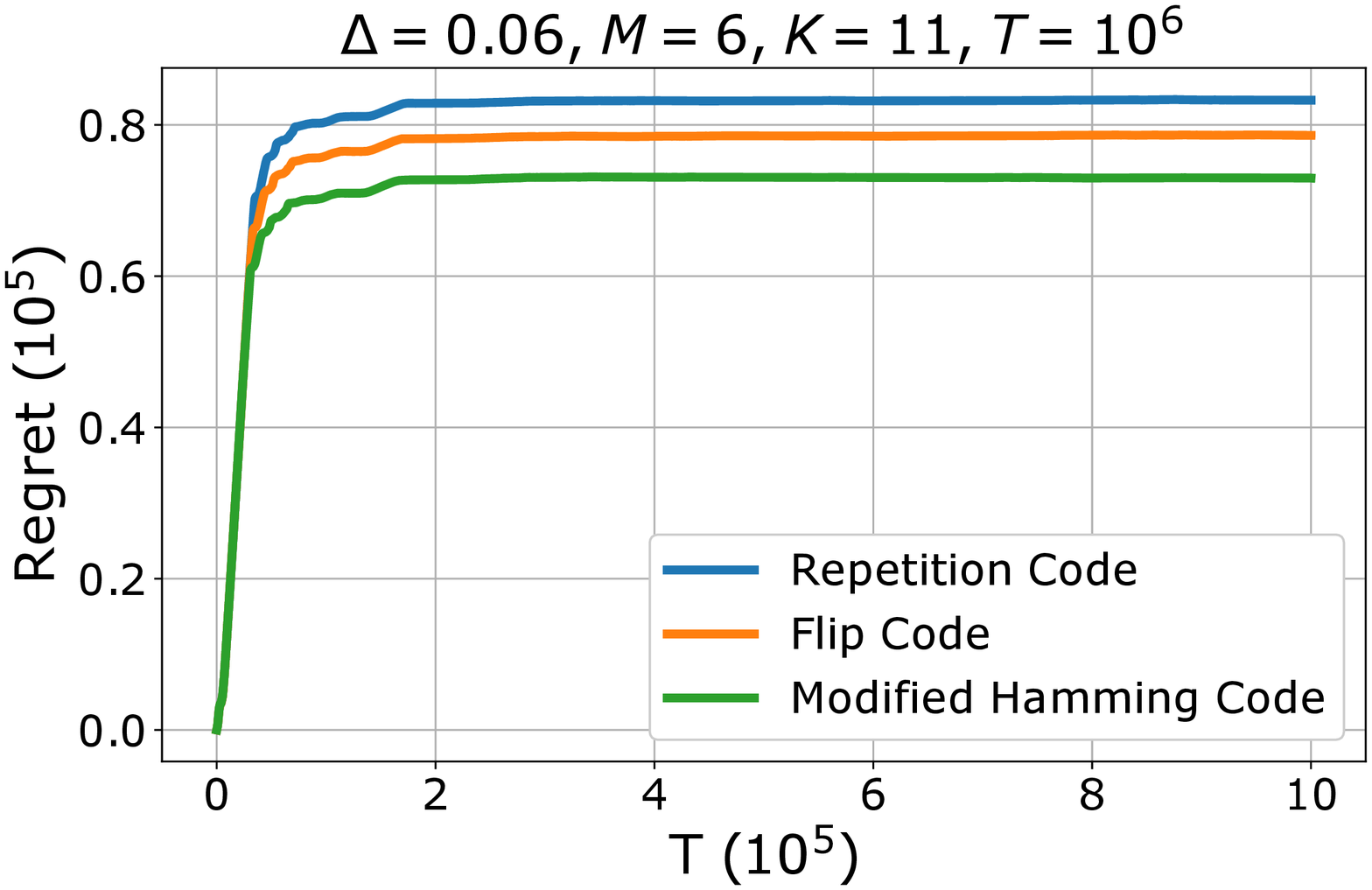}
		\caption{Different coding techniques}
		\label{fig:coding}
	\end{minipage}
	\begin{minipage}[t]{0.32\textwidth}
		\centering
		\includegraphics[width=5.5cm,height=4cm]{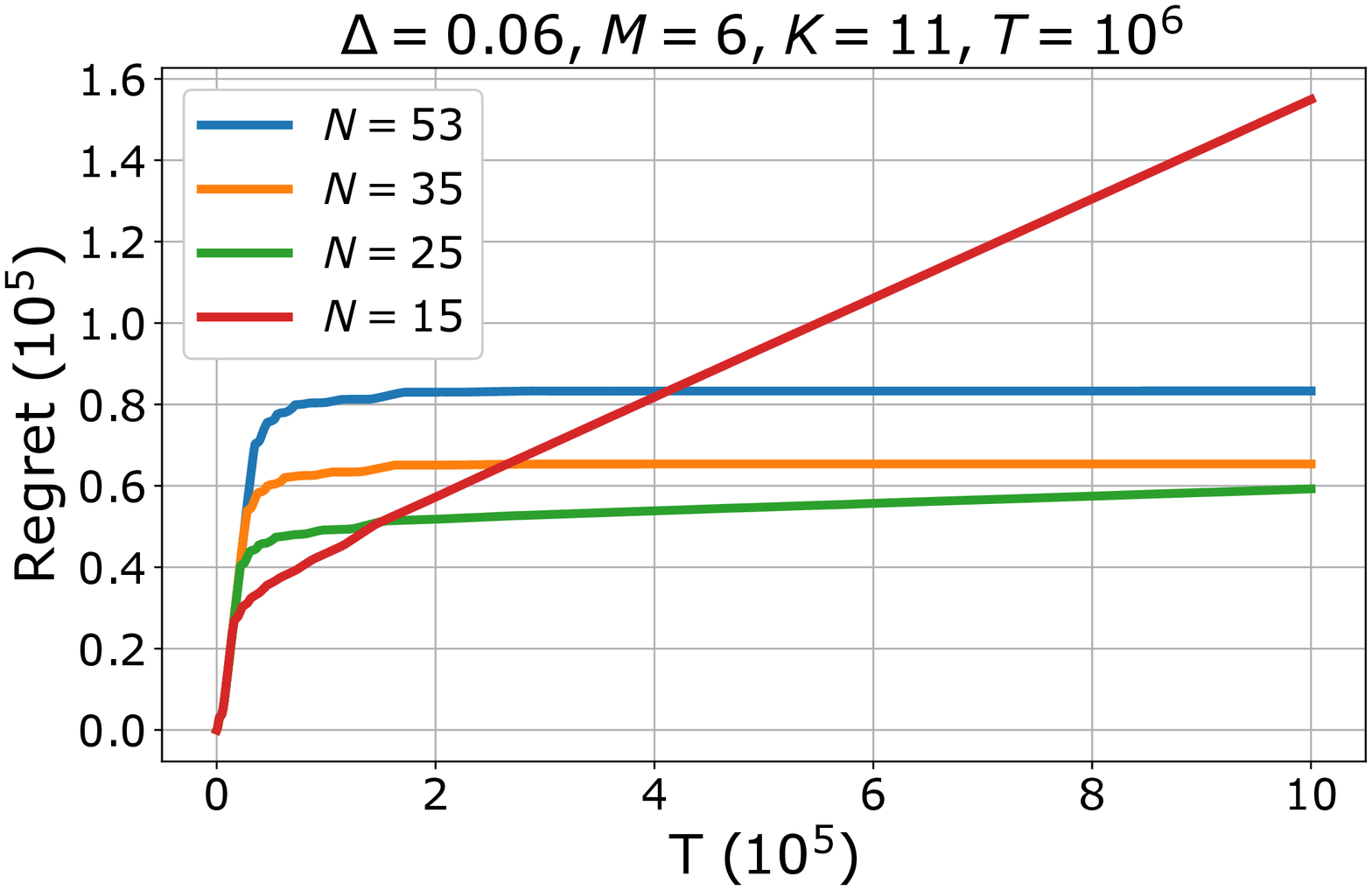}
		\caption{Different codeword lengths}
		\label{fig:length}
	\end{minipage}
	\begin{minipage}[t]{0.32\textwidth}
		\centering
		\includegraphics[width=5.5cm,height=4cm]{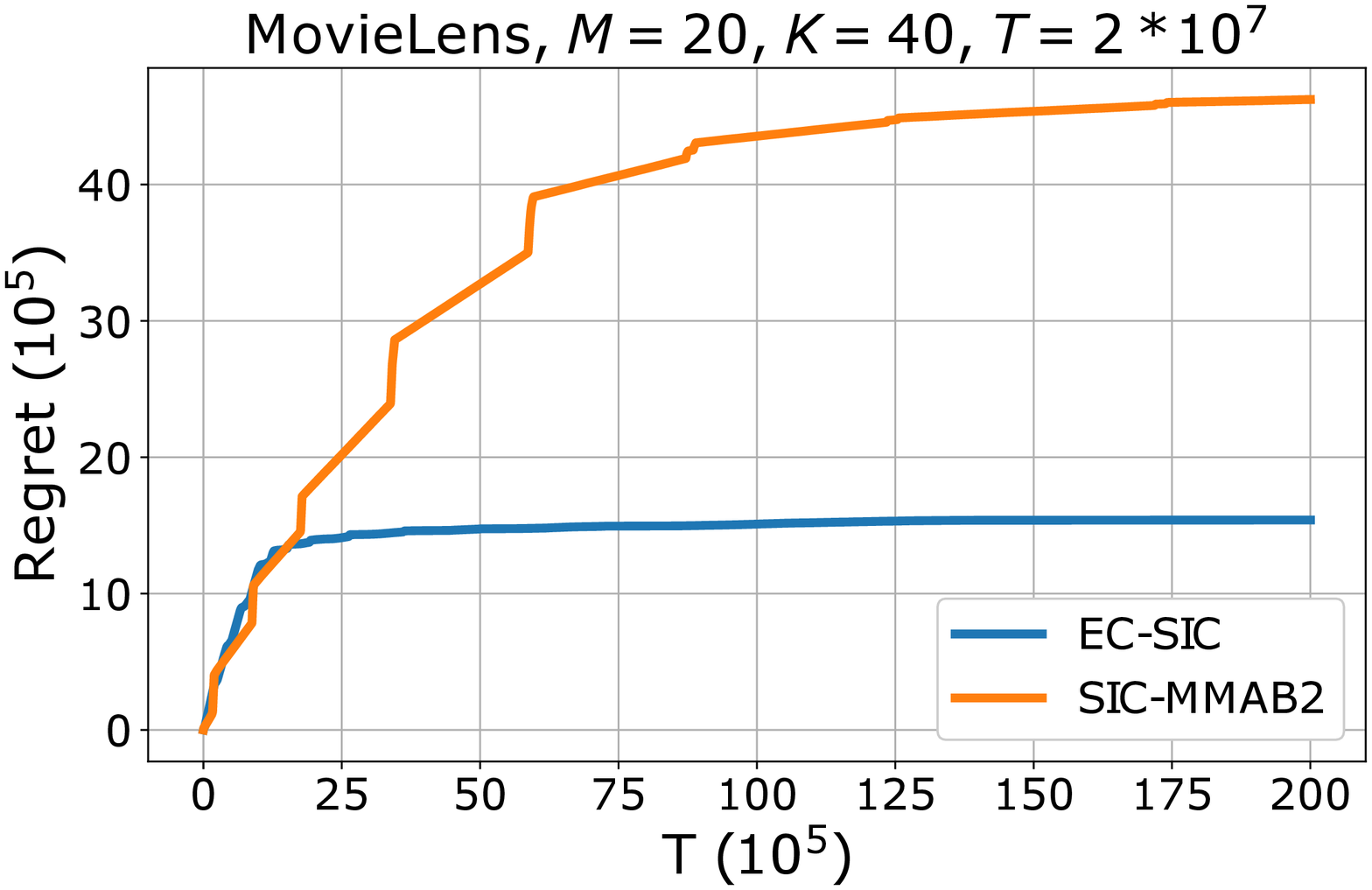}
		\caption{The MovieLens dataset}
		\label{fig:movie}
	\end{minipage}
    \end{figure*}
    
    	\vspace{-0.1in}
	\section{Algorithm Enhancement}
	\label{sec:enhance}
	EC-SIC  has nice theoretical performance guarantees, but we have noticed that in practice some minor enhancements improve its performance significantly, which are \congr{shown in the next section}.  First, after each exploration and communication phase, player $j$ can use the active arm with the $j$-th best empirical mean (sent by the leader) as her communication arm for the next around, instead of the $j$-th active arm. Since players keep receiving rewards from their communication arms while waiting for communication or receiving bit $0$, using an arm with higher empirical mean can lead to a lower loss in these time steps. 
	
	Second, we have observed in practice that the first one or two explorations do not lead to effective acceptation or rejection even when the game is easy, which means all the communication losses during these phases are incurred with no benefit (this is much larger than the exploration loss). Thus, $p$ can be initialized to a larger integer (e.g. $5$), which leads to a longer exploration to start with and less ineffective communication. 
	
	\cssr{Lastly, if $\mu_{\min}$ and $\Delta$ in Assumption \ref{asp} are not available, \emph{adaptive estimation with confidence intervals} can be used to replace the true $\mu_{\min}$ and $\Delta$ in EC-SIC. The influence of mismatched $\mu_{\min}$ and $\Delta$ are evaluated in the experiments and reported in Appendix \ref{appendix:exp}.}
	
	\vspace{-0.1in}
	\section{Experiments}
	\label{sec:sim}
	

    Numerical experiments have been carried out to verify the analysis of EC-SIC and compare its empirical performance to other methods. All rewards follow the Bernoulli distributions with $\mu_{\min}=0.3$, and we set $\epsilon=\Delta/8$. Results are obtained by averaging over 500 experiments. More detailed discussions and additional results can be found in Appendix~\ref{appendix_sim}.
    
    We compare state of the art algorithms under both easy and difficult bandit game settings. EC-SIC (with repetition code), ADAPTED SIC-MMAB, SIC-MMAB2, and the algorithm proposed by \cite{lugosi2018multiplayer} (labeled as ``no-sensing-MC'') are first compared in a relatively easy game ($\Delta = 0.06$). Fig.~\ref{fig:easy_game} shows that even in an easy game, no-sensing-MC could not finish exploration within $10^6$ time steps, and ADAPTED SIC-MMAB has poor performance compared to the other two. Both EC-SIC and SIC-MMAB2 converge to the optimal arm set quickly, but the overall regret of EC-SIC is smaller. For a hard game with $\Delta = 0.01$, Fig.~\ref{fig:hard_game} shows that EC-SIC is superior to SIC-MMAB2. 
    
    A detailed comparison of EC-SIC with SIC-MMAB2 is done by comparing their regrets as a function of the gap $\Delta$ in Fig.\ref{fig:game_difficulty}. We see that when the game is not extremely difficult ($\Delta > 10^{-4}$), EC-SIC has better performance since players benefit from sharing statistics. When $\Delta$ becomes extremely small, the required communication length increases significantly, leading to a dominating communication regret in EC-SIC that cannot be offset by the benefits of sharing statistics. 
	
	
	Fig.~\ref{fig:coding} reports the performance while using different Z-channel codes in communication.  We observe that \cssr{modified} Hamming Code has the best performance, which is due to its superior error correction capability. This observation also implies that with a near-optimal code that is specifically designed for Z-channel, performance of EC-SIC can be further improved.
	
	
	We also evaluate the impact of \congr{codeword} length on the regret. For our simulation setting, the theoretical analysis requires a repetition code length $N = 53$ to transmit one bit, in order to achieve an error rate of $\frac{1}{T}$. We are interested in evaluating whether the theoretically required code length can be shortened in practice. Under the easy game setting of Fig.~\ref{fig:easy_game} with $2000$ rounds averaging, Fig.~\ref{fig:length} shows that with $N$ decreasing from $53$ to $35$, the regret decreases $20\%$. More importantly, it shows that the convergence of EC-SIC does not change. When further reducing $N$ to $25$, we see the regret curve trends upward at large $t$, which represents a non-negligible loss due to unsuccessful communications. With $N=15$, the regret increases rapidly, indicating that players suffer from an increased error rate. It is thus essential to strike a balance between error rate and communication loss.
	
	Lastly, we evaluate EC-SIC on \congr{a real world dataset}: the movie watching dataset (ml-20m) from MovieLens \citep{harper2016movielens}. It consists of watching data of more than $2\times 10^4$ movies from over $10^5$ users between January 09, 1995 and March 31, 2015. \congr{In} the pre-processing, we group these movies into $K=40$ \congr{categories} by their total number of views from high to low. The \congr{binary} reward at time $t$ (hour) is \congr{defined as} whether there are users watching films in this group, and we replicate it to a final reward sequence of length $T=2\times 10^7$. $M=20$ players are assumed to engage in the game. This final sequence has $\Delta\approx 0.007$ and $\mu_{\min}\approx 0.6$. \cssr{Compared to synthetic datasets, this setting poses a larger and more difficult game.} For each \congr{experiment}, the reward sequence is randomly shuffled. We report the \congr{cumulative} regret of EC-SIC and SIC-MMAB2, \congr{averaged over $100$ experiments,} in Figure \ref{fig:movie}. One can see that the advantage of EC-SIC over SIC-MMAB2 is significant for this real-world dataset. Intuitively, this is because the game is hard ($\Delta\approx 0.007$)\footnote{However, the game is also not too hard for communication to be ineffective as the case of Fig.~\ref{fig:game_difficulty}.}, and $M$ and $K$ are also large.
	
	
    
	
	\vspace{-0.1in}
	\section{Related Work}
	\label{sec:related}
	
	Depending on how information is shared and actions are determined, existing literature can be categorized into centralized or decentralized (distributed) MP-MAB problems.  The centralized scenario can be viewed as an application of the multiple-play bandit \citep{anantharam1987asymptotically,komiyama2015optimal}. A more interesting and challenging problem, introduced by \cite{liu2010distributed} and \cite{anandkumar2011distributed}, lies in the decentralized scenario where explicit communications between players are not allowed and thus collision may happen. For the collision-sensing MP-MAB problem, earlier works attempt to let each player play the single-player MAB game while avoiding collisions for as much as possible; see \citep{liu2010distributed,avner2014concurrent,rosenski2016multi} for some representative approaches. 

	
	The SIC-MMAB algorithm in \cite{Boursier2019} is closely relevant to our work, which proposes to \emph{exploit} collisions as opposed to \emph{avoiding} them.   
	\cssr{\cite{proutiere2019optimal} further refines this idea and decreases the communication regret so that the lower bound of the centralized setting can be approached asymptotically for Bernoulli distributed rewards.} The SIC-MMAB principle has subsequently been applied to other multi-player settings. For example, \cite{kalathil2014decentralized} considers an extended multiplayer model where reward distribution varies for each player. \cite{bistritz2018distributed} proposes a Game of Thrones algorithm that achieves a regret of $O(\log^{2+\kappa}(T))$. This is further improved by \cite{kaufmann2019new}, leading to an improved regret of $O(\log^{1+\kappa}(T))$.
	
	The no-sensing model, on the other hand, is very challenging and  limited progress has been made so far. In \cite{lugosi2018multiplayer}, sample means are rectified by the probability of collision and then the same Musical Chair approach is adopted. As discussed in Section~\ref{sec:intro}, \cite{Boursier2019} touches upon the no-sensing model with ADAPTED SIC-MMAB and SIC-MMAB2. However, the former has a communication loss of $O(\log(T)\log^2(\log(T)))$ that dominates the total regret, while the latter drifts away from communication of full statistics, thus is fundamentally incapable of approaching the centralized performance.
	
	\vspace{-0.1in}
	\section{Conclusion}
	\label{sec:conc}
	In this work, we have proposed the EC-SIC algorithm for the no-sensing MP-MAB problem with forced collision. We proved that it is possible for a decentralized MP-MAB algorithm without collision information to approach the performance of its centralized counterpart. Recognizing that communication under \congr{the} no-sensing setting corresponds to the Z-channel model in information theory, optimal error correction codes are applied for reliable communication via collision. With this tool, we return to the original idea of utilizing forced collisions to share complete arm statistics among players. By expanding exploration phases and fixing the message length, an order-optimal communication loss is achieved. Practical simulation results with several Z-channel codes have proved the superiority of EC-SIC algorithm under different bandit game settings\congr{, using both synthetic and real-world datasets}. 
	
	\clearpage

    \subsubsection*{Acknowledgements}
    JY acknowledges the support from U.S. National Science Foundation under Grant ECCS-1650299.

	\medskip
	\small
	\bibliographystyle{apalike}
	\bibliography{MAB}
	

	\newpage
	\onecolumn
	\appendix
	\thispagestyle{empty}
	
	\hsize\textwidth
    \linewidth\hsize \toptitlebar {\centering
    {\Large\bfseries Supplementary Material: Decentralized Multi-player Multi-armed Bandits with No Collision Information \par}}
    \bottomtitlebar
    
    \aistatsauthor{Chengshuai Shi \And Wei Xiong \And Cong Shen \And Jing Yang}
    
    \appaddress{University of Virginia \And University of Virginia \And University of Virginia  \And Pennsylvania State University}
    
    \vspace{0.2in}
    

	\section{Error Correction Codes for Communication over the Z-channel}
	\label{appendix_coding}
	More details about the representative coding techniques for the Z-channel are provided in this section.
	\subsection{Repetition code}
	Repetition code seems simple but is surprisingly powerful in the Z-channel. \cite{chen2013optimal} has proved that it is the optimal code for $Q=1$. The encoding and decoding processes are described as follows.
	\begin{itemize}
		\item Encoding. Repeat bit $0$ or bit $1$ in message $\boldsymbol{m}$ for $A$ times to generate codeword $\boldsymbol{X}$.
		\item Decoding. For channel output $\boldsymbol{Y}$, if there exists $i$ such that $\boldsymbol{Y}[i]\not=0$, then the decoder outputs 0. Otherwise, we have $\boldsymbol{Y}[i]=0$ for all $i$, and the decoder outputs 1.
	\end{itemize}
	With a crossover probability {no larger than $1-\mu_{\min}$}, the bit error probability is:
	\begin{equation*}
	P(Y_i \not= X_i) <  (1-\mu_{\min})^A.
	\end{equation*}
	For a message length of $Q$ bits, the error probability is:
	\begin{equation*}
	\begin{aligned}
	P_e &= P(\exists i, Y_i\not = X_i) \\ 
	& = 1-P(Y_i = X_i)^Q \\ 
	& \leq 1-(1-(1-\mu_{\min})^A)^Q \\ &\leq Qe^{-\mu_{\min}A}.
	\end{aligned}
	\end{equation*}
	With the choice of $A = \lceil \frac{\log(QT)}{\mu_{\min}}\rceil$, we have $P_e<\frac{1}{T}$. Thus, the total code length for a $Q$-bit message is:
	\begin{equation*}
	N_{rep} = Q \left \lceil \frac{\log(QT)}{\mu_{\min}} \right \rceil.
	\end{equation*}
	With $N_{rep}=\Theta(\log(T))$, the regret remains order-optimal.
	
	\subsection{Flip code}
	The flip code is designed by \cite{chen2013optimal} to better utilize the Z-channel property. The encoding and decoding processes are illustrated with the case of 4 codewords as follows.
	\begin{itemize}
		\item Encoding. Assuming we encode every two bits into a $2A$-bit codeword, the encoding function is:
		$$(0,0)\to(\underbrace{1,...,1}_{A},\underbrace{1,...,1}_{A});\  (0,1)\to(\underbrace{1,...,1}_{A},\underbrace{0,...,0}_{A});$$
		$$(1,0)\to(\underbrace{0,...,0}_{A},\underbrace{1,...,1}_{A});\  (1,1)\to(\underbrace{0,...,0}_{A},\underbrace{0,...,0}_{A}).$$
		\item Decoding. It is similar to the repetition code. A codeword $\boldsymbol{m}$ of length $2A$ will be divided into $\boldsymbol{m_1}$ of length $A$ and $\boldsymbol{m_2}$ of length $A$
		\begin{itemize}
			\item if all bits in $\boldsymbol{m_1}$ and $\boldsymbol{m_2}$ are $1$s, decoder outputs $(0,0)$;
			\item if all bits in $\boldsymbol{m_1}$ are $1$s and $\boldsymbol{m_2}$ contains $0$, decoder outputs $(0,1)$;
			\item if $\boldsymbol{m_1}$ contains $0$ and all bits in $\boldsymbol{m_2}$ are $1$s, decoder outputs $(1,0)$;
			\item for all other cases, decoder outputs $(1,1)$.
		\end{itemize}
	\end{itemize}
	With a crossover probability {no larger than $1-\mu_{\min}$}, the bit error probability is \citep{chen2013optimal}:
	\begin{equation*}
	P(Y_i\not=X_i) \leq (1-\mu_{\min})^A-\frac{1}{4}(1-\mu_{\min})^{2A}
	\end{equation*}
	The inequality holds because the function $q^A-\frac{1}{4}q^{2A}$ monotonically increases for $q\in [0,1]$. For a message length of $Q$ bits (we assume $Q$ is even here, otherwise an additional bit $0$ can always be padded to make it even), the error probability is:
	\begin{equation*}
	\begin{aligned}
	P_e &= P(\exists i, Y_i\not = X_i)\\
	& = 1-P(Y_i = X_i)^{\frac{Q}{2}}\\
	& \leq 1-(1-(1-\mu_{\min})^A+\frac{1}{4}(1-\mu_{\min})^{2A})^{\frac{Q}{2}} \\
	& = 1 - (1-\frac{1}{2}(1-\mu_{\min})^A)^Q \\ &\leq \frac{Q}{2}(1-\mu_{\min})^A \\ &\leq \frac{Q}{2}e^{-\mu_{\min}A}.
	\end{aligned}
	\end{equation*}
	With the choice of $A = \lceil \frac{\log(QT/2)}{\mu_{\min}}\rceil$, we have $P_e <\frac{1}{T}$. Thus, the total codeword length for a message of length $Q$ is:
	\begin{equation*}
	N_{flip} = Q \lceil \frac{\log(QT/2)}{\mu_{\min}}\rceil.
	\end{equation*}
	With $N_{flip}=\Theta(\log(T))$, the regret remains order-optimal.
	
	\subsection{Modified Hamming code}
	As the number of codewords increases to $16$ (4 bits), a modified ($7$,$4$) Hamming Code can be designed. It is a concatenated code, with the standard ($7$,$4$) Hamming code as the inner code and  a repetition  code as the outer code. 
	\begin{itemize}
		\item Encoding. The standard (7,4) Hamming encoding matrix $\boldsymbol{G}$ is first used to encode a 4-bit message to a 7-bit codeword. Then we repeat each bit of the 7-bit codeword $A$ times, leading to a $7A$-bit codeword;
		\item Decoding. First by using the repetition code's decoding rule, $7A$-bit coded message is decoded into $7$ bits. This $7$ bits is then decoded with the standard (7,4) Hamming decoding matrix $\boldsymbol{H}$. The final output is a decoded $4$-bit message.
		$$ \boldsymbol{G}=
		\left(
		\begin{matrix}
		1& 1& 0& 1\\
		1& 0& 1& 1\\
		1& 0& 0& 0\\
		0& 1& 1& 1\\
		0& 1& 0& 0\\
		0& 0& 1& 0\\
		0& 0& 0& 1
		\end{matrix}
		\right),\ 
		\boldsymbol{H}=
		\left(
		\begin{matrix}
		1& 0& 1& 0& 1& 0& 1\\
		0& 1& 1& 0& 0& 1& 1\\
		0& 0& 0& 1& 1& 1& 1
		\end{matrix}
		\right)
		$$
	\end{itemize}
	The repetition code reduces the crossover probability from $q$ to $q^A$. With this relatively small crossover probability and the error correction capability of the Hamming Code, a reliable performance can be achieved. As stated by \cite{barbero2006maximum}, with $q^A$ as the crossover probability, we have the following error rate for the Hamming Code over a Z-channel.
	\begin{equation*}
		P(Y_i\not= X_i) = \frac{7}{2}(q^A)^2+o((q^A)^3).
	\end{equation*}
We neglect $o((q^A)^3)$ in the following analysis. The error probability of transmitting $Q$-bit messages (assuming $Q$ can be divided by $4$) using the ($7$,$4$) modified Hamming code is:
	\begin{equation}
	\label{eqn:ham_bound}
	\begin{aligned}
		P_e &= P(\exists i, Y_i\not = X_i)\\
		& = 1-P(Y_i = X_i)^{\frac{Q}{4}}\\
		& = 1-(1-\frac{7}{2}q^{2A})^{\frac{Q}{4}} \\ &\leq 1-(1-\frac{7}{2}(1-\mu_{\min})^{2A})^{\frac{Q}{4}}\\
		& \leq \frac{7Q}{8}(1-\mu_{\min})^{2A} \\ &\leq \frac{7Q}{8} e^{-2\mu_{\min}A}.
	\end{aligned}
	\end{equation}
	By choosing $A = \frac{1}{2}\lceil \frac{\log(7QT/8)}{\mu_{\min}}\rceil$, we have $P_e < \frac{1}{T}$. Thus, the total codeword length for a message of length $Q$ is:
	\begin{equation*}
		N_{ham} = \frac{7Q}{8} \left \lceil \frac{\log(7QT/8)}{\mu_{\min}} \right \rceil,
	\end{equation*}
	which is still $\Theta(\log(T))$, but the bound in \eqref{eqn:ham_bound} indicates an improvement over the repetition code and flip code.
	
	\section{Proofs for the Regret Analysis}
	\label{appendix_proof}
	\subsection{Initialization phase}
	\label{appendix_proof_init}
	\cssr{The initialization phase starts with a ``Muscial Chair'' phase, which assigns a unique external rank in ${1,...,K}$ for each of the player. Then the following sequential hopping protocol converts the external rank into a unique internal rank in ${1,...,M}$ for each player and estimates the number of players $M$.} The proof for Lemma~\ref{regret_init_lemma} is the same as  Lemma 11 in  \cite{Boursier2019}. We re-state the algorithm and the proof for the sake of completeness.
	\begin{algorithm}[htb]
		\small
		\caption{Musical\_Chair}
		\label{mc}
		\begin{algorithmic}[1]
			\Require [$K$], $T_0$
			\Ensure Fixed (external rank)
			\State Initialize Fixed $\gets -1$
			\For {$T_0$ time steps}
			\If {Fixed $=-1$} 
			\State Sample $k$ uniformly at random in [$K$], play it in round $t$ \cssr{and receive reward $r(t)$}
			\If {$\cssr{r(t)}>0$} Fixed $\gets k$
			\EndIf
			\EndIf
			\EndFor
			\State \Return {Fixed}
		\end{algorithmic}
	\end{algorithm}
	
	\begin{algorithm}[htb]
		\small
		\caption{Estimate\_M\_NoSensing}
		\label{estimate_m}
		\begin{algorithmic}[1]
			\Require $k$ (external rank), $T_c$
			\Ensure $M$, $j$ (internal rank)
			\State Initialize $M \gets 1$ and $\pi \gets k$
			\For {$n=1,2,...,2K$}
			\State $r \gets 0$
			\If {$n \leq 2k$}
			\For {$T_c$ time steps} 
			\State Pull $\pi$ and get reward $r_{\pi}(t)$
			\State Update $r \gets r+r_{\pi}(t)$
			\EndFor
			\If {r=0} $M \gets M+1$, $j \gets j+1$
			\EndIf
			\Else \ $\pi \gets \pi+1 (\text{mod }K)$ 
			\For {$T_c$ time steps} 
			\State Pull $\pi$ and get reward $r_{\pi}(t)$
			\State Update $r \gets r+r_{\pi}(t)$
			\EndFor
			\If {r=0} $M \gets M+1$ 
			\EndIf
			\EndIf
			\EndFor
			\State \Return {$M$, $j$}
		\end{algorithmic}
	\end{algorithm}
	
	\begin{proof}
		As there is at least one arm that is not played by all the other players, the probability to encounter a positive reward for player $j$ during the Musical Chair process at time $t$ is lower bounded by $\frac{\mu_{\min}}{K}$. Thus with the choice of $T_0 = K\lceil \frac{\log(T)}{\mu_{\min}} \rceil$, the probability for a single player to encounter only zero rewards until time $T_0$ is:
		\begin{equation*}
		P(\forall t\leq T_0, \cssr{r^j(t)=0}) \leq (1-\frac{\mu_{\min}}{K})^{T_0}\leq \exp(-\frac{T_0\mu_{\min}}{K}) \leq \frac{1}{T}.
		\end{equation*}
		Applying a union bound over all players, the Musical Chair process is successful with a probability at least $1-O(\frac{M}{T})$.
		
		Now we analyze the \emph{Estimate\_M\_NoSensing} protocol. Similar to using repetition code for communication, the probability that a player detects a ``collision'' while there is none is:
		\begin{equation*}
		P_e \leq (1-\mu_{\min})^{T_c} \leq e^{-\mu_{\min} T_c} \leq \frac{1}{T}.
		\end{equation*}
		The union bound over the $M$ players and the $2K$ blocks yields that it will be successful with a probability at least $1-O(\frac{MK}{T})$. Furthermore, the initialization phase lasts $3KT_c$ time steps. Hence the regret satisfies:
		\begin{equation*}
		R^{init}\leq 3MKT_c = 3MK\lceil\frac{\log(T)}{\mu_{\min}}\rceil.
		\end{equation*}
	\end{proof}
	
	\subsection{Exploration phase}
	This section aims at proving Lemma~\ref{regret_explo_lemma}, which bounds the exploration regret. We start with the required lemmas and then go back to proving Lemma~\ref{regret_explo_lemma}.
	\subsubsection{Proof for Lemma~\ref{regret_good_comm}}
	This lemma ensures \cssr{that event $A_2$ happens} with a high probability. As mentioned before, there are at most $\log_2(T)$ communication phases, which leads to at most $(MK+2M)\log_2(T)$ instances of transmissions to send arm statistics to the leader and send the acc/rej arm sets to the followers. Since there are at most $K$ arms to be accepted or rejected, no more than $MK$ instances of transmissions are required for sending the acc/rej arm sets.
	\begin{proof}
		Denote $P(\cssr{\xi_p})$ as the probability that the decoding of a $Q$-bit message produces a wrong result at round \cssr{$p$}. With the choice of $N'=\max\{\frac{Q}{\cssr{C_Z(1-\mu_{\min})}}, \frac{\log(T)}{E(R)}\}$, and \cssr{$X_p$, $Y_p$} denoting the message before encoding and after decoding at round $\cssr{p}$, we have
		\begin{equation*}
		P(\xi_\cssr{p}) = P(\cssr{X_p \not= Y_p}) \leq \frac{1}{T}.
		\end{equation*}
		A simple union bound leads to
		\begin{equation*}
		P_r = 1- P(\cup \cssr{p}, \xi_\cssr{p}) \geq 1-\sum_{\cssr{p}}P(\xi_\cssr{p}) \geq 1-\frac{(MK+2M)\log(T)+MK}{T}\cssr{=1-O\left(\frac{MK\log(T)}{T}\right)}.
		\end{equation*}
	\end{proof}
	
	\subsubsection{Proof for Lemma~\ref{regret_correct_ar}}
	Lemma~\ref{regret_correct_ar} ensures the acceptance and rejection of arms are successful with a high probability, which requires a good estimation of the statistics of arms. The estimation error consists of two parts: the quantization error and the sampling error. We analyze them separately.
	\begin{proof}
		With the choice of $Q \geq \log_2(\frac{1}{\frac{\Delta}{4}-\epsilon})$, the quantization error \cssr{in phase $p$} can be bounded as:
		\begin{equation*}
		\begin{aligned}
		&\cssr{\forall i\in[M]}, |\bar{\mu}_i^p[k]-\hat{\mu}_i^p[k]| \leq \frac{\Delta}{4}-\epsilon,\\
		&|\bar{\mu}^p[k]-\hat{\mu}^p[k]| = \frac{ \left |\sum_{i=1}^M(\bar{\mu}_i^p[k]-\hat{\mu}_i^p[k])\cssr{T_p^i} \right |}{T_p}\leq \frac{\Delta}{4}-\epsilon.
		\end{aligned}
		\end{equation*}
		\cssr{For any active arm $k\in[K_p]$,} the gap between the sample mean $\hat{\mu}^p[k]$ (using all players' samples) and the true mean can be bounded with Hoeffding's inequality:
		\begin{equation*}
		\cssr{P\left (|\hat{\mu}^p[k]-\mu[k]|\geq\sqrt{\frac{2\log(T)}{T_p}} \right)\leq \frac{2}{T}.
		}
		\end{equation*}
		Then, the overall gap between the quantized mean and the true mean \cssr{for any active arm $k\in [K_p]$} can be bounded as:
        \cssr{
        \begin{equation*}
		\begin{aligned}
		&P\left (|\bar{\mu}^p[k]-\mu[k]|>B_{T_p} \right )\\
		 = & P \left (|\bar{\mu}^p[k]-\hat{\mu}^p[k]+\hat{\mu}^p[k]-\mu[k]|\geq\sqrt{\frac{2\log(T)}{T_p}}+\frac{\Delta}{4}-\epsilon \right )\\
		 \leq & P\left (|\bar{\mu}^p[k]-\hat{\mu}^p[k]|+|\hat{\mu}^p[k]-\mu[k]|\geq\sqrt{\frac{2\log(T)}{T_p}}+\frac{\Delta}{4}-\epsilon \right )\\
		\leq & P\left (|\hat{\mu}^p[k]-\mu[k]|\geq \sqrt{\frac{2\log(T)}{T_p}})\cup P(|\bar{\mu}^p[k]-\hat{\mu}^p[k]|\geq\frac{\Delta}{4}-\epsilon \right )\\
		 = & P\left (|\hat{\mu}^p[k]-\mu[k]|\geq \sqrt{\frac{2\log(T)}{T_p}} \right ) \\
		 \leq & \frac{2}{T}.
		\end{aligned}
		\end{equation*}
	There are at most $\log_2(T)$ iterations of exploration and communication.
	By using a union bound of all these iterations and $K$ arms, Eqn. \eqref{equation_corret_ar} is obtained.}
	\end{proof}
	
	\subsubsection{Proof for Lemma~\ref{regret_artimes_lemma}}
	Lemma~\ref{regret_artimes_lemma} bounds the number of time steps an arm is pulled before being accepted or rejected, and is essential to control the rounds of exploration and communication. The proof is similar to the proof to Proposition 1 in \cite{Boursier2019}.
	\begin{proof}
		The proof is conditioned on \cssr{the typical event.}
		We first consider an optimal arm $k$. Let $\Delta_k = \mu[k]-\mu_{(M+1)}$ be the gap between the arm $k$ and the first sub-optimal arm. Let $s_k$ be the first integer such that $4B_{s_k}\leq \Delta_k$. It satisfies:
		\begin{equation*}
		s_k\geq\frac{32\log(T)}{  \left (\Delta_k-\Delta+4\epsilon \right )^2}=\frac{32\log(T)}{ \left (\mu[k]-\mu_{(M)}+4\epsilon \right )^2}.
		\end{equation*}
		
		Recall that the number of time steps an active arm is pulled before the $p$-th exploration is $T_p = \sum_{l=1}^{p}M_l2^l\lceil \log(T)\rceil$. With a non-increasing $M_p$, it holds that 
		\begin{equation}
		\label{tp}
		T_{p+1}\leq 3T_p.
		\end{equation}
		
		For some $p$ such that $T_{p-1}\leq s_k<T_p$, the following inequalities are in order: $\Delta_{k} \geq 4B_{T_p}$; $|\bar\mu^p[k]-\mu[k]|\leq B_{T_p}$; and $|\bar\mu^p[i]-\mu[i]|\leq B_{T_p}$ for all sub-optimal arm $i$.
		We then have
		\begin{equation*}
		\bar\mu^p[k]-B_{T_p}\geq \bar\mu^p[\cssr{i}]+B_{T_p}+\mu[k]-\mu[i]-4B_{T_p}\geq \bar\mu^p[i]+B_{T_p}.
		\end{equation*}
		This suggests arm $k$ will be accepted at time $T_p$. Eqn. \eqref{tp} also leads to $T_p = O(s_k) = O  \left  (\frac{\log(T)}{(\mu[k]-\mu_{(M)}+4\epsilon)^2} \right )$. Thus, arm $k$ will be accepted after at most $O  \left  (\frac{\log(T)}{(\mu[k]-\mu_{(M)}+4\epsilon)^2} \right )$ pulls. The part of rejecting sub-optimal arms can be similarly proved with  $\Delta_k = \mu_{(M)}-\mu[k]$.
	\end{proof}
	
	\subsubsection{Lemma~\ref{regret_decom_lemma} and its proof}
	\label{appd:lem6}
	\begin{lemma}
		\label{regret_decom_lemma}
		In the typical event, the following \congr{results hold.}
		\vspace{-0.07in}
		\begin{equation}
		\notag
		\begin{aligned}
		&\text{1) \congr{for any sub-optimal arm} $k$, }(\mu_{(M)}-\mu[k])T^{expl}_{k}(T) = O \left (\frac{\Delta}{4\epsilon}\min\left\{\frac{\log(T)}{\mu_{(M+1)}-\mu[k]+4\epsilon},\sqrt{T\log(T)}\right\} \right );
		\\&\text{2) }\sum_{k\leq M}(\mu_{(k)}-\mu_{(M)})(T^{expl}-T^{expl}_{(k)}) = 
		O \left (\cssr{\sum_{k>M}}\min\left\{\frac{\log(T)}{\mu_{(M+1)}-\mu_{(k)}+4\epsilon},\sqrt{T\log(T)}\right\} \right).
		\end{aligned}
		\end{equation}
	\end{lemma}
	
	The proof of the first part in Lemma~\ref{regret_decom_lemma} is as follows.
	\begin{proof}
		For a sub-optimal arm $k$, Lemma~\ref{regret_artimes_lemma} leads to $T_k^{expl}(T)\leq O \left (\min\left\{\frac{\log(T)}{(\mu_{(M+1)}-\mu[k]+4\epsilon)^2},T\right\} \right )$, and thus
		\begin{equation*}
		\begin{aligned}
		(\mu_{(M)}-\mu[k])T^{expl}_{k}(T)&=\frac{\mu_{(M)}-\mu[k]}{\mu_{(M+1)}-\mu[k]+4\epsilon}O\left (\min\left\{\frac{\log(T)}{(\mu_{(M+1)}-\mu[k]+4\epsilon)},(\mu_{(M+1)}-\mu[k]+4\epsilon)T\right\} \right)\\
		&\overset{(i)}{\leq} O\left(\frac{\Delta}{4\epsilon}\min\left\{\frac{\log(T)}{\delta},\delta T\right\}\right)\\
		&\overset{(ii)}{\leq} O\left (\frac{\Delta}{4\epsilon}\min\left\{\frac{\log(T)}{(\mu_{(M+1)}-\mu[k]+4\epsilon)},\sqrt{T\log(T)}\right\} \right),
		\end{aligned}
		\end{equation*}
		in which inequality (i) comes from $$\frac{\mu_{(M)}-\mu[k]}{\mu_{(M+1)}-\mu[k]+4\epsilon} = \frac{\mu_{(M)}-\mu[k]}{\mu_{(M)}-\mu[k]+4\epsilon-\Delta} \leq \frac{\Delta}{4\epsilon},$$
		and $\delta = \mu_{(M+1)}-\mu[k]+4\epsilon$. Inequality (ii) can be obtained with the observation that the term $\frac{\Delta}{4\epsilon}O(\min\{\frac{\log(T)}{\delta},\delta T\})$ is maximized by $\delta=\sqrt{\frac{\log(T)}{T}}$.
	\end{proof}
	
	The second part of Lemma~\ref{regret_decom_lemma} is based on Lemmas~\ref{appendix_lemma_part1} and \ref{appendix_lemma_part2}.
	\begin{lemma}
		\label{appendix_lemma_part1}
		Define $\hat{t}_k$ as the number of exploratory pulls before accepting/rejecting the arm $k$ and \cssr{$H$ is the total number of exploration phases}. \cssr{Conditioned on the typical event}, we have:
		\begin{equation*}
		\sum_{k\leq M} \left (\mu_{(k)}-\mu_{(M)} \right )  \left (T^{expl}-T^{expl}_{(k)} \right ) \leq \sum_{j>M}\sum_{k\leq M}\sum_{p=1}^\cssr{H} 2^p \lceil \log(T) \rceil  \left  (\mu_{(k)}-\mu_{(M)} \right )\mathds{1}_{\min \left\{\hat{t}_{(j)},\hat{t}_{(k)} \right \}\geq T_{p-1}}.
		\end{equation*}
	\end{lemma}
	\begin{proof}
		For an optimal arm $k$, during phase $p$, if $k$ has already been accepted, it will be pulled $K_p2^p\lceil \log(T) \rceil$ times. If it is still active (i.e., $\hat{t}_k>T_{p-1}$), it will be pulled $M_p2^p\lceil \log(T) \rceil$ times, meaning that this arm is not pulled for $(K_p-M_p)2^p\lceil \log(T) \rceil$ times. 
		Thus, it holds that $T^{expl}_k\geq T^{expl}-\sum_{p=1}^\cssr{H} 2^p(K_p-M_p)\cssr{\lceil\log(T)\rceil}\mathds{1}_{\hat{t}_k> T_{p-1}}$. Notice that $K_p-M_p = \sum_{j>M}\mathds{1}_{\hat{t}_{(j)}> T_{p-1}}$. We have $T^{expl}_k\geq T^{expl}-\sum_{p=1}^\cssr{H} \sum_{j>M}2^p\cssr{\lceil\log(T)\rceil}\mathds{1}_{\min\{\hat{t}_{(j)},\hat{t}_{(k)}\}>T_{p-1}}$, which proves the lemma. 
	\end{proof}
	
	\begin{lemma}
		\label{appendix_lemma_part2}
		\cssr{Conditioned on the typical event}, we have:
		\begin{equation*}
		\sum_{k\leq M}\sum_{p=1}^\cssr{H} 2^p \lceil \log(T) \rceil \left (\mu_{(k)}-\mu_{(M)} \right )\mathds{1}_{\min\left \{\hat{t}_{(j)},\hat{t}_{(k)} \right \}\geq T_{p-1}}\leq O\left (\min\left \{\frac{\log(T)}{\mu_{(M)}-\mu_{(j)}+4\epsilon},\sqrt{T\log(T)} \right \} \right ).
		\end{equation*}
	\end{lemma}
	
	\begin{proof}
		Define $A_j = \sum_{k\leq M}\sum_{p=1}^\cssr{H} 2^p \lceil \log(T) \rceil(\mu_{(k)}-\mu_{(M)})\mathds{1}_{\min\{\hat{t}_{(j)},\hat{t}_{(k)}\}\geq T_{p-1}}$. Notice that $$\hat{t}_{(k)} \leq \min\left\{\frac{c\log(T)}{(\mu_{(k)}-\mu_{(M)}+4\epsilon)^2},T\right\},$$
		and denote $\Delta(p)=\sqrt{\frac{c\log(T)}{T_{p-1}}}$. The inequity $\hat{t}_{\cssr{(k)}}>T_{p-1}$ implies $\mu_{(k)}-\mu_{(M)}<\Delta(p)-4\epsilon$. We also denote $N^j$ as the smallest integer satisfying $\hat{t}_{(j)} \leq T_{N^j}$. Then we have the following:
		\begin{equation*}
		\begin{aligned}
		A_j & \leq \sum_{k\leq M}\sum_{p=1}^{N^j} 2^p \lceil \log(T) \rceil(\Delta(p)-4\epsilon)\mathds{1}_{\hat{t}_{\cssr{(k)}}\geq T_{p-1}}\\
		&\leq\sum_{p=1}^{N^j}\Delta(p)2^p\lceil \log(T) \rceil\sum_{k\leq M}\mathds{1}_{\hat{t}_{\cssr{(k)}}\geq T_{p-1}} \\ 
		&= \sum_{p=1}^{N^j}\Delta(p)2^p\lceil \log(T) \rceil M_p\\
		&\leq \sum_{p=1}^{N^j}\Delta(p)(T_p-T_{p-1}) \\
		&=c\log(T) \sum_{p=1}^{N^j}\Delta(p)\left(\frac{1}{\Delta(p+1)}+\frac{1}{\Delta(p)}\right)\left (\frac{1}{\Delta(p+1)}-\frac{1}{\Delta(p)} \right ).
		\end{aligned}
		\end{equation*}
		Since $T_{p+1}\leq 3T_p$, $\Delta(p)\left (\frac{1}{\Delta(p+1)}+\frac{1}{\Delta(p)} \right )=1+\sqrt{\frac{T_p}{T_{p-1}}}\leq 1+\sqrt{3}$. Thus, 
		$$A_j\leq c\log(T)\sum_{p=1}^{N^j}\left (\frac{1}{\Delta(p+1)}-\frac{1}{\Delta(p)} \right )\leq (1+\sqrt{3})c\log(T)\frac{1}{\Delta(N^j+1)}.$$
		With the definition of $N^j$, we have $\cssr{\hat{t}_{(j)}}\geq T_{N^j-1}$. With inequality $T_{N^j+1}\leq 3T_{N^j}$ we have $\Delta(N^j)\geq \sqrt{\frac{c\log(T)}{3\cssr{\hat{t}_{(j)}}}}$. $A_j\leq (3+\sqrt{3})\sqrt{c\cssr{\hat{t}_{(j)}}\log(T)}$ then holds. With $\cssr{\hat{t}_{(j)}\leq }   O\left (\min\{\frac{c\log(T)}{(\mu_{(M+1)}-\mu_{(j)}+4\epsilon)^2}, T\} \right )$, we have 
		$$A_j \leq(3+\sqrt{3})\min \left \{\frac{c\log(T)}{\mu_{(M+1)}-\mu_{(j)}+4\epsilon},\sqrt{cT\log(T)} \right\}.$$
	\end{proof}
	
	\subsection{Communication phase}
	This section presents the proof related to the bound of the communication regret. 
	\subsubsection{Proof for Lemma~\ref{regret_comm_lemma}}
	\begin{proof}
		\cssr{Conditioned on the typical event}, we denote $H$ as the number of exploration phases. The communication length for sending arm statistics and acc/rej arm sets for $p\in [H]$ is at most $N'(KM+2M)$. Lemma~\ref{regret_artimes_lemma} states that $H$ satisfies $T_H = \sum_{l=1}^H M_l 2^l \lceil \log(T) \rceil =O(\cssr{\max_{k\in[K]}\{s_k\}})=O \left (\min\{\frac{\log(T)}{4\epsilon},T\} \right )$. Thus,
		\begin{equation*}
		H = O\left (\log(\min\{\frac{1}{4\epsilon},\cssr{T}\}) \right),
		\end{equation*}
		which leads to a regret of $O(N'(KM^2+2M^2)\log(\min\{\frac{1}{4\epsilon},T\})$. Next, transmitting acc/rej arm sets at most incurs a regret of \cssr{$M^2KN'$}. Putting them together, the total communication regret is:
		\begin{equation*}
		O\left( N'(KM^2+2M^2)\log\left (\min \left \{\frac{1}{4\epsilon},\cssr{T} \right \} \right)+\cssr{N'}M^2K \right).
		\end{equation*}
	\end{proof}
	
	\section{Supplementary Materials for Experiments}
	\label{appendix_sim}
	\subsection{Minor changes to SIC-MMAB2}
	Some minor changes can be made to SIC-MMAB2 \citep{Boursier2019} to improve its convergence and make it more practical. First, the original length of exploration in phase $p$ is $K_p2^p T_0$, where $T_0=\lceil \frac{2400\log(T)}{\mu_{\min}}\rceil$. The factor $2400$ is too large and makes it almost impossible to converge (or even finish one round of exploration) in the time horizon of our experiments. We thus change it to $T_0=\lceil \frac{24\log(T)}{\mu_{\min}}\rceil$, and we have verified that it converges successfully for the instances in the experiments. Another minor change is to add a random selection in the declaration phases.  In SIC-MMAB2 \citep{Boursier2019}, players sequentially declare the arms in their own \cssr{rejected} sets. However, with similar exploration time steps across players, the declaration sets are almost identical for all players. With a sequential selection (rather than a random selection in our implementation), all players will declare the same arm with high probability, which is in fact noticed to happen very frequently in the experiment. 
	
	\subsection{Supplementary experiment results}\label{appendix:exp}
	
	\begin{figure*}[htb]
	\begin{minipage}{0.5\textwidth}
		\setlength{\abovecaptionskip}{-2pt} 
        \centering
        \includegraphics[width=5.5cm,height=4cm]{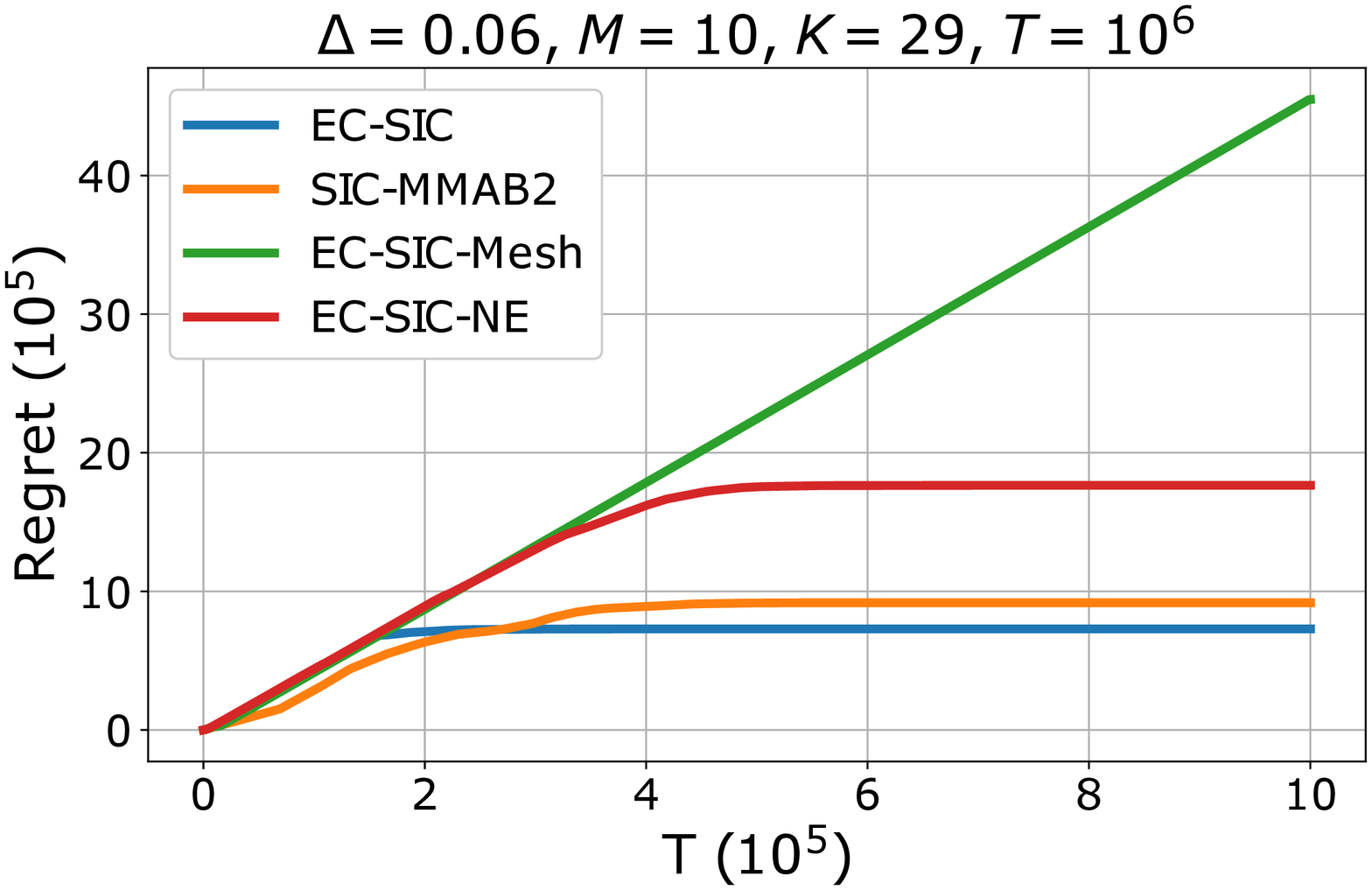}
		\caption{Large game}
		\label{fig:large_game}
	\end{minipage}
	\begin{minipage}{0.5\textwidth}
		\setlength{\abovecaptionskip}{-2pt} 
        \centering
        \includegraphics[width=5.5cm,height=4cm]{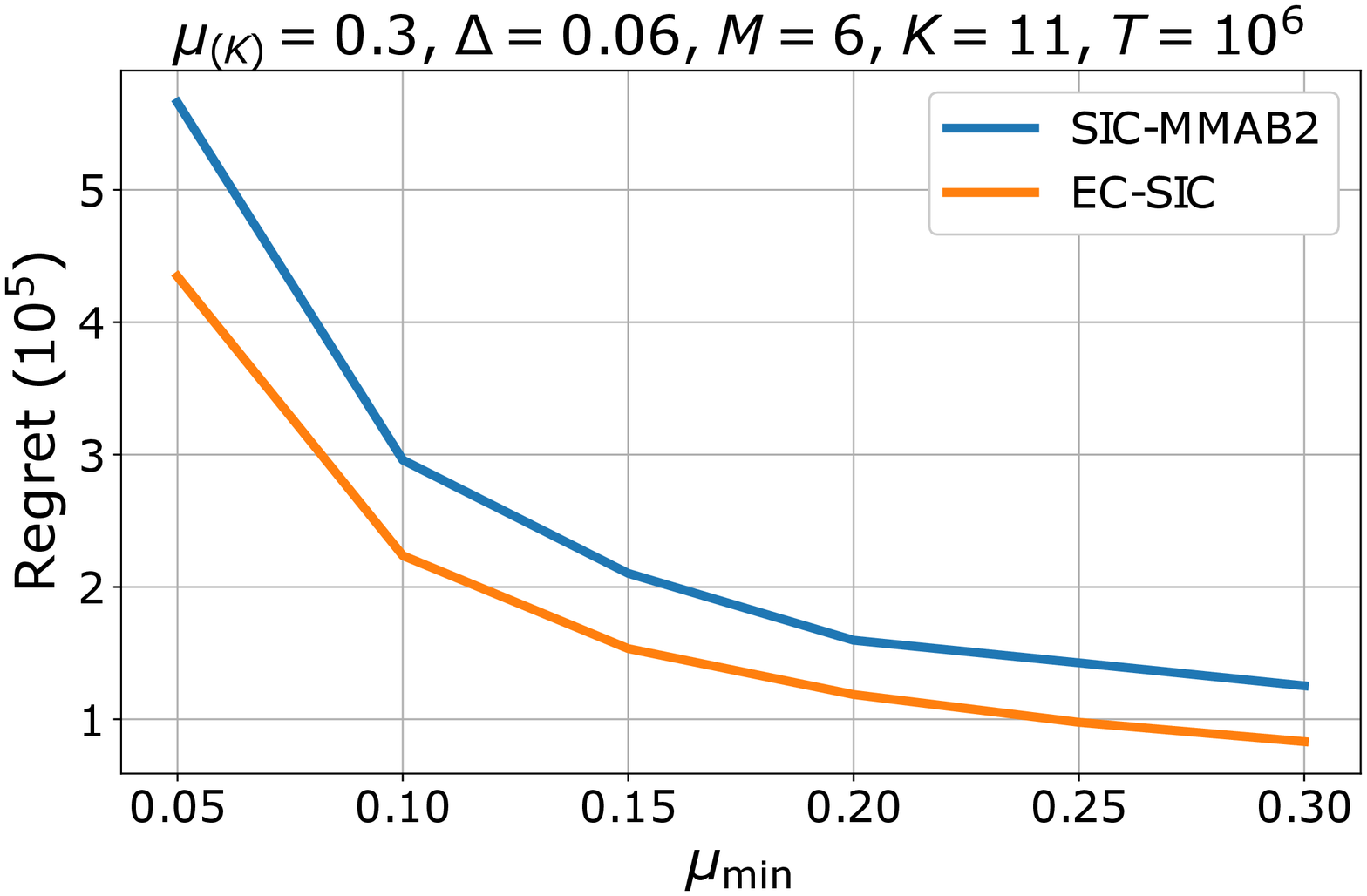}
	    \caption{Performance with different $\mu_{\min}$}
	    \label{fig:mu}
	\end{minipage}
    \end{figure*}
    
    The impact of the leader-follower protocol and the enhancement in Section~\ref{sec:enhance} are evaluated in a larger game ($M=10$, $K=29$). With $\Delta = 0.06$, EC-SIC without the leader-follower protocol (labeled as ``EC-SIC-Mesh''), EC-SIC using none of the enhancement in Section \ref{sec:enhance} (labeled as ``EC-SIC-NE'') are compared with EC-SIC and SIC-MMAB2. Compared to the stable performance of EC-SIC and SIC-MMAB2, Fig.~\ref{fig:large_game} shows that in practice, due to the unnecessary communication between every pair of players, EC-SIC-Mesh cannot even finish the first communication phase. Although it converges eventually, EC-SIC-NE has a larger regret. These results highlight the benefit of the tree-structured communication and the selection of better communication arms.
    
    
    In the case of utilizing repetition code in EC-SIC, we carry on experiments to evaluate the dependency on the knowledge of $\mu_{\min}$. For $\mu_{(k)}=0.3$, we evaluate EC-SIC and SIC-MMAB2 with $\mu_{\min}$ from $0.05$ to $0.3$, hence creating a mismatched ``estimation'' of $\mu_{\min}$. The results shown in Figure \ref{fig:mu} state that decreasing $\mu_{\min}$ leads to an increasing regret of both SIC-MMAB2 and EC-SIC, which corroborates the theoretical analysis. Furthermore, EC-SIC has better performance than SIC-MMAB2 across different ``estimates'' of  $\mu_{\min}$.

    
    The knowledge of $\Delta$ is assumed in the algorithms and their theoretical analysis. In practice, a precise value of $\Delta$ may not always be available. In the last experiment, we demonstrate the robustness of the algorithms by feeding it with inaccurate information $\Delta_e$ instead of the true $\Delta$. As shown in Figure \ref{fig:pes_delta}, with a pessimistic estimation of $\Delta$, the inaccurate information only leads to some additional but acceptable communication loss. The overall regret is still better than SIC-MMAB2. For the optimistic estimations, Figure \ref{fig:opt_delta} shows that the algorithm is effective even with $\Delta_e=2\Delta$. When the estimation error further grows ($\Delta_e=3 \Delta$ or $6  \Delta$), communication errors start to occur, which lead to a few arms that are identified incorrectly. It nevertheless still outperforms SIC-MMAB2  in terms of regret. Thus, practically speaking, EC-SIC has good robustness, and we further comment that it is preferable to have a pessimistic estimation. 

    \begin{figure*}[htb]
	\begin{minipage}{0.5\textwidth}
		\setlength{\abovecaptionskip}{-2pt} 
        \centering
        \includegraphics[width=5.5cm,height=4cm]{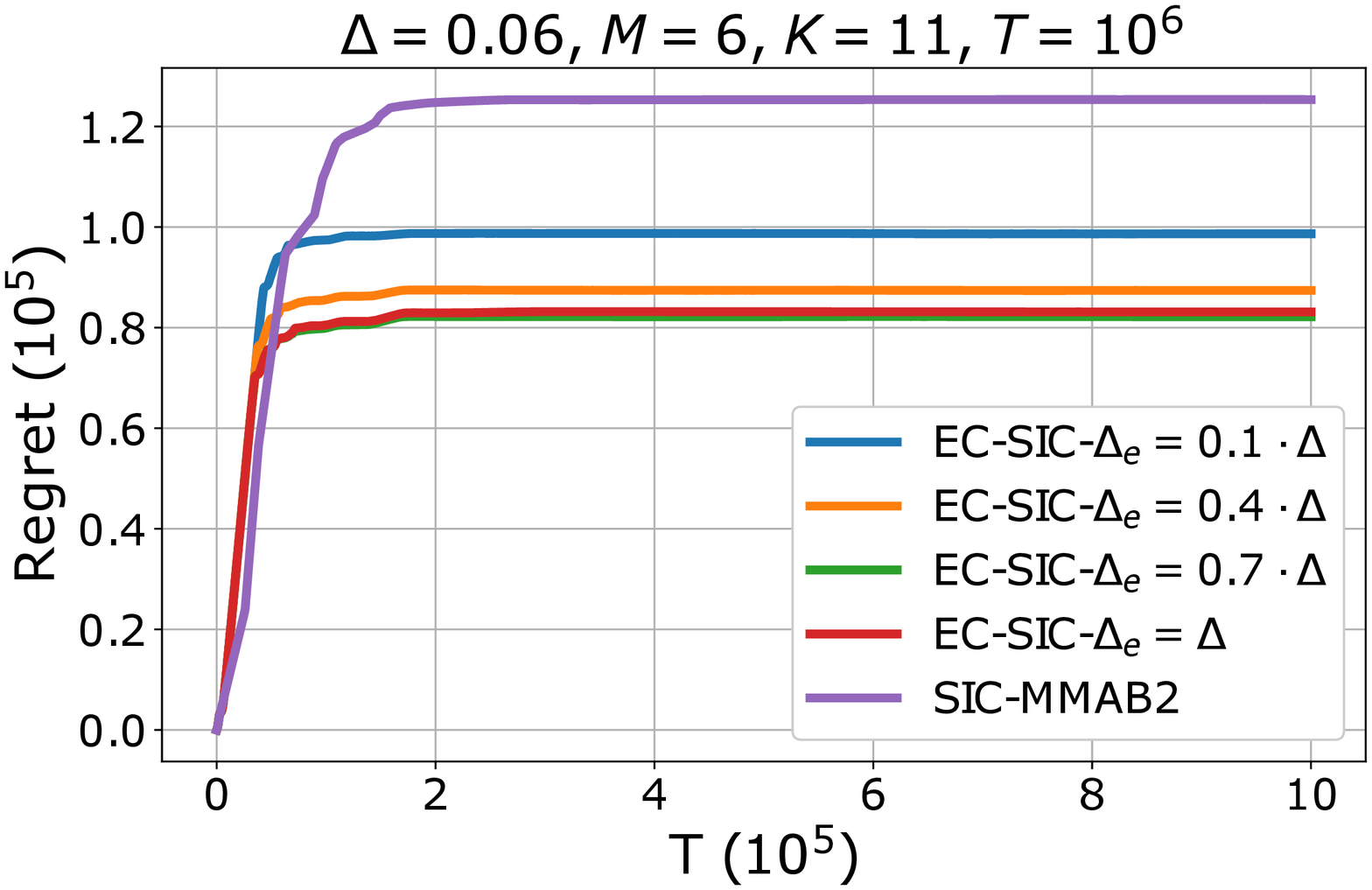}
		\caption{Pessimistic estimation}
		\label{fig:pes_delta}
	\end{minipage}
	\begin{minipage}{0.5\textwidth}
		\setlength{\abovecaptionskip}{-2pt} 
        \centering
        \includegraphics[width=5.5cm,height=4cm]{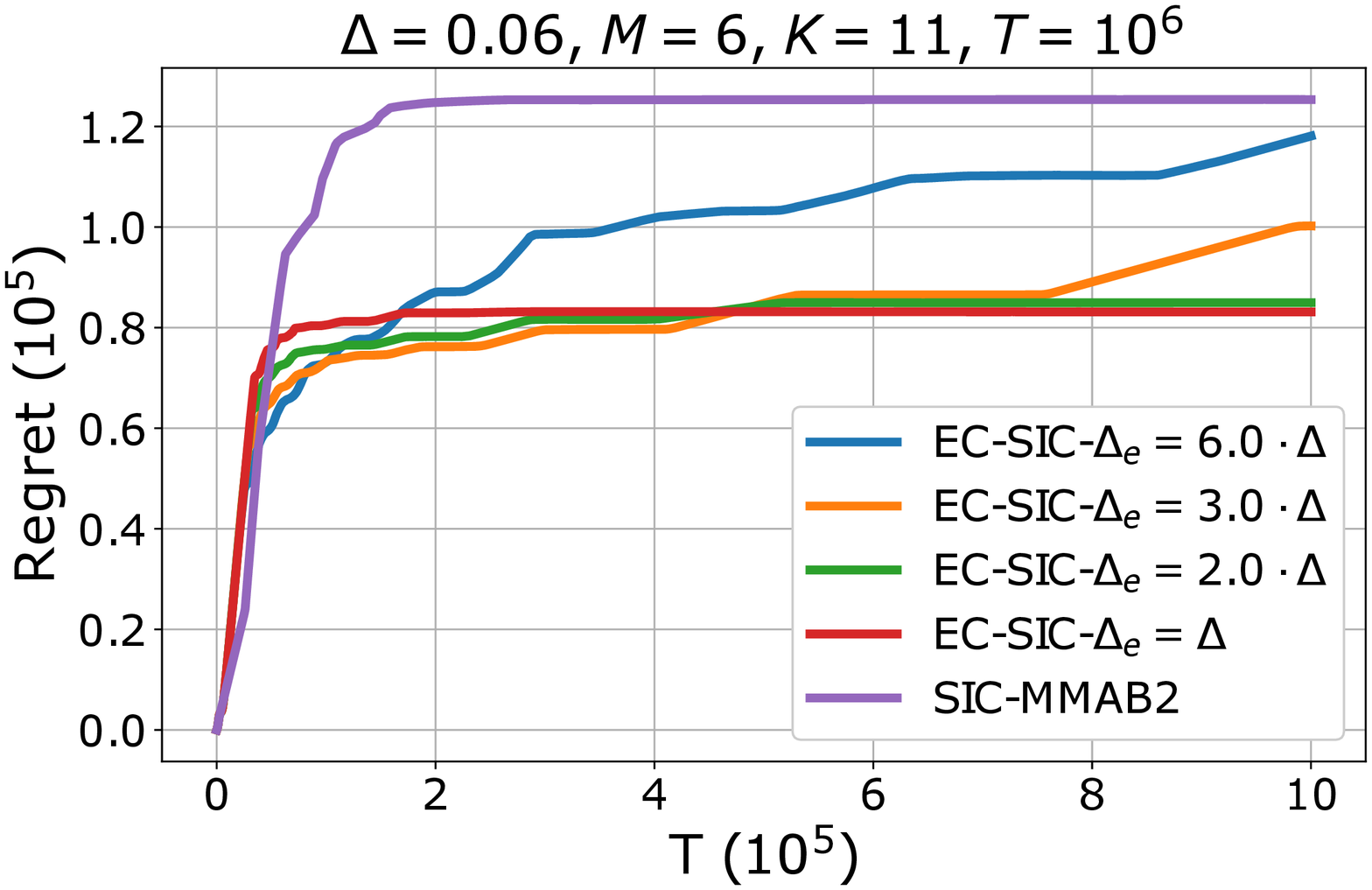}
		\caption{Optimistic estimation}
		\label{fig:opt_delta}
	\end{minipage}
    \end{figure*}
    
\end{document}